\documentclass[english]{article}
\usepackage[T1]{fontenc}
\usepackage[latin9]{inputenc}
\usepackage{color}
\usepackage{verbatim}
\usepackage{amsthm}
\usepackage{amsmath}
\usepackage{amssymb}
\usepackage{graphicx}
\usepackage{esint}

\makeatletter
\theoremstyle{plain}
\newtheorem{thm}{\protect\theoremname}
  \theoremstyle{plain}
  \newtheorem{lem}[thm]{\protect\lemmaname}

\usepackage{proceed2e}

\usepackage{algorithmic}
\usepackage{algorithm}

\usepackage{appendix}
\usepackage{amsmath}
\usepackage{amsfonts}
\usepackage{amsthm}
\usepackage{enumerate}

\usepackage[ plainpages = true, pdfpagelabels, 
             pdfpagelayout = useoutlines,
             bookmarks,
             bookmarksopen = true,
             bookmarksnumbered = true,
             breaklinks = true,
             linktocpage,
             pagebackref,
             colorlinks = true,
             linkcolor = blue,
             urlcolor  = blue,
             citecolor = blue!50!black,
             anchorcolor = green,
             hyperindex = true,
             hyperfigures
             ]{hyperref} 
\usepackage[round]{natbib}
\usepackage[resetlabels,labeled]{multibib}
\usepackage{soul}
\usepackage{subfig}
\usepackage{times}

\usepackage[textsize=small]{todonotes}

\usepackage{url}
\usepackage{wrapfig}

\newcommand{\diag}{\mathop{\mathrm{diag}}}

\newcommand{\bx}{\mathbf{x}}				% all variables
\newcommand{\factor}{f}				% AG: made this f for consistency
\newcommand{\outV}{V}                         %AG: output variable of a factor
\newcommand{\fis}[1]{\mathrm{ne}(#1)}   	% index set for variables connected to  factor
\newcommand{\fx}[1]{ \mathbf{x}_{\mathrm{ne}(#1)} }   	% variables of a factor
\newcommand{\xin}{\mathbf{x}_{ \mathrm{in} }} 			% parents of directed factor
\newcommand{\xout}{\mathbf{x}_{ \mathrm{out} }}			% child of directed factor
\newcommand{\msg}[2]{m_{#1 \rightarrow #2}}			% message from arg1 to arg2
\newcommand{\approxMsg}[3]{M_{#1 \rightarrow #2}^{#3}}			% message from arg1 to arg2
\newcommand{\uncert}{{\mathfrak v}}          %%AG: variable used to denote uncertainty
\newcommand{\uncertaintyMsg}[3]{{\mathfrak V}_{#1 \rightarrow #2}^{#3}}			% message from arg1 to arg2
\newcommand{\diffd}{\mathrm{d}}

\newcommand{\projP}[1]{\mathrm{proj} \left [ #1 \right]}   
\newcommand{\argmin}[1]{\mathrm{arg}\mathrm{min}_{#1}}
\newcommand{\kld}[2]{\mathrm{KL} \left [ #1 || #2 \right ]}
\newcommand{\expectationE}[2]{ \mathbb{E}_{#2}  \left[ #1 \right] }

\newcommand{\feax}{\mathsf{x}}
\newcommand{\feaX}{\mathsf{X}}
\newcommand{\feay}{\mathsf{y}}
\newcommand{\feaY}{\mathsf{Y}}

\newcommand{\wjnote}[1]{ }
\newcommand{\aenote}[1]{}
\newcommand{\nhnote}[1]{}
\newcommand{\agnote}[1]{}
 \newcommand{\dsnote}[1]{}%\textbf{\color{orange!50!black}{#1}} }

\newcommand{\figref}[1]{Fig.~\ref{#1}}
\newcommand{\secref}[1]{Section~\ref{#1}}

\newcites{sup}{References}

\makeatother

\usepackage{babel}
  \providecommand{\lemmaname}{Lemma}
\providecommand{\theoremname}{Theorem}

\allowdisplaybreaks

\begin{document}
\title{Kernel-Based Just-In-Time Learning for \\Passing Expectation Propagation
Messages}
% \author{Wittawat Jitkrittum, \, Arthur Gretton \\
% Gatsby Unit, University College London \\
% \url{wittawatj@gmail.com}, \, \url{arthur.gretton@gmail.com}
% \And Nicolas Heess, \, S. M. Ali Eslami \\
% \url{nheess@nhuk.de} \\
% \url{ali@arkitus.com}
% %Google DeepMind
% \AND Balaji Lakshminarayanan \\
% Gatsby Unit, University College London \\
% \url{balaji@gatsby.ucl.ac.uk}
% \And Dino Sejdinovic \\
% University of Oxford \\
% \url{dino.sejdinovic@gmail.com}
% \And Zolt{\'a}n Szab{\'o} \\
% Gatsby Unit, University College London \\
% \url{z.szabo@ucl.ac.uk}
% }

\author{
    Wittawat Jitkrittum,$^1$ \, Arthur Gretton,$^1$ \, Nicolas
    Heess,\thanks{\hspace{1mm} Currently at Google DeepMind.} \, \,
\textbf{S. M. Ali Eslami}$^*$ \\
\textbf{Balaji Lakshminarayanan},$^1$ \, \textbf{Dino Sejdinovic}$^2$ \and
\textbf{ Zolt{\'a}n Szab{\'o}}$^1$  \vspace*{2mm} \\
Gatsby Unit, University College London$^1$ \\
\vspace*{2mm}
University of Oxford$^2$ \\
\texttt{ \{wittawatj,  arthur.gretton\}@gmail.com}, \, \texttt{nheess@gmail.com} \\ 
\texttt{ali@arkitus.com}, \, \texttt{balaji@gatsby.ucl.ac.uk},\\
\texttt{dino.sejdinovic@gmail.com},\, \texttt{zoltan.szabo@gatsby.ucl.ac.uk}
}

\maketitle

\begin{abstract}

  We propose an efficient nonparametric strategy for learning a message operator in expectation propagation (EP), which takes as input the set of incoming messages to a factor node, and produces an outgoing message as output. This learned operator replaces the multivariate integral required in classical EP, which may not have an analytic expression. We use kernel-based regression, which is trained on a set of probability distributions representing the incoming messages, and the associated outgoing messages. The kernel approach has two main advantages: first, it is fast, as it is implemented using a novel two-layer random feature representation of the input message distributions; second, it has principled uncertainty estimates, and can be cheaply updated online, meaning it can request and incorporate new training data when it encounters inputs on which it is uncertain.  In experiments, our approach is able to solve learning problems where a single message operator is required for multiple, substantially different data sets (logistic regression for a variety of classification problems), where
it is essential to
   accurately assess uncertainty and to efficiently and robustly update the message operator.

\end{abstract}

\section{INTRODUCTION}

%% Intro paragraphs

%AG: not sure where these shoudl go.
%\cite{Rezende2014,Kingma2013,Stuhlmuller2013,Ross2011}

An increasing priority in Bayesian modelling is to make inference accessible and implementable for practitioners,
without  requiring specialist knowledge.
% in Bayesian inference.
This
is a goal sought, for instance, in probabilistic programming languages
\citep{WinGooStuSis11,GooManRoyBonTen08}, 
as well as in more granular, component-based
%AG: at Ali's suggestion, removed these two references
%procedures \citep{DuvLloGroetal13,GroSalFreTen12} and
systems \citep{SDT2014,Minka2014}. In all cases, the user
should be able to freely specify what they wish their model to express,
without having to deal with the complexities of sampling, variational
approximation, or distribution conjugacy.  In reality, however, model convenience and
simplicity can  limit
or undermine intended models, sometimes in ways the users
might not expect. To take one example, the inverse gamma prior, which is
widely used as a convenient
conjugate prior for the variance, has quite pathological behaviour \citep{Gelman2006}.
In general, more expressive, freely chosen models are more likely
to require expensive sampling or quadrature approaches, which can make
them challenging to implement or impractical to run.

We address the particular setting of expectation propagation \citep{Minka2001}, a message
passing algorithm wherein messages are confined to being members of a particular parametric
family. The process of integrating incoming messages over a factor potential, and projecting the result onto the
required output family, can be difficult, and in some cases not achievable in closed form.
Thus, a number of approaches have been proposed to implement EP updates numerically, independent
of the details of the factor potential being used.  One approach, due to
\citet{Barthelme2011}, is to compute the message update via importance sampling.
While these estimates converge to the desired integrals for a sufficient number of importance samples, the sampling procedure must be run at every iteration during inference,
hence  it is not viable for large-scale problems.

An improvement on this approach is to use importance sampled instances of input/output
message pairs to train a regression algorithm, which can then be used in place of the sampler.
\citet{Heess2013} use neural networks to learn the mapping from incoming
to outgoing messages, and the learned mappings perform well on a variety of practical problems.
This approach comes with a disadvantage: it requires training data
that cover the entire set of possible input messages for a given type of problem (e.g., datasets
representative of all classification problems the user proposes to solve),
and it has no way of assessing the uncertainty of its prediction, or of updating
the model online in the event that a prediction is uncertain.

%AG: I shortened the introduction and moved details to later
The disadvantages of the neural network approach were the basis for work by
\cite{Eslami2014}, who replaced the neural networks with random forests.
The random forests  provide uncertainty estimates for each prediction. This allows them to be trained `just-in-time', during EP inference, whenever the predictor decides it is uncertain.
Uncertainty estimation for random forests relies on unproven heuristics, however: we demonstrate
empirically that such heuristics can become highly misleading as we move away from the initial
training data.
Moreover,  online updating
can result in unbalanced trees, resulting in a cost of prediction of $O(N)$ for training data of size $N$, rather than the ideal of $O(\log(N))$.

% AG: costs of GP are in paper ``Fast Gaussian Process Regression using KD-Trees'', which gives
% yet another way to speed up GP prediction.

We propose a novel, kernel-based approach to learning a message operator nonparametrically
for expectation propagation. The learning algorithm takes the form of a distribution regression
problem, where the inputs are probability measures represented as embeddings
of the distributions to a reproducing kernel Hilbert space (RKHS), and the outputs are vectors of message
parameters \citep{Szabo2014}. 
A first advantage of this approach is that one does not need to pre-specify customized features
of the distributions, as in \citep{Eslami2014,Heess2013}. Rather, 
we use a general characteristic kernel on input distributions 
\citep[eq. 9]{Christmann2010}.
%, which in our experiments gives better performance than customized features. 
%\wjnote{We do not have such an experiment comparing the two.}
%
%AG: Actually, we do have this experiment: you also compared with Gaussian kernels on the mean and precision, eg.
%This was a a table that was going to go in appendix. But I am shortening the intro so I cut this detail.
%
%A potential downside of the kernel approach is that it can be computationally
%costly, with training time of $O(N^3)$ and a cost of $O(N)$ to make a prediction. 
To make the algorithm computationally tractable, 
we regress  directly in the primal from random Fourier features of the data \citep{Rahimi2007,Le2013,YanSmoZonWil14}.
In particular, we establish a novel random feature representation for when
 inputs are distributions, via a  two-level random feature approach.
This gives us both fast prediction (linear in the number of
random features), and fast online updates (quadratic in the number of random features).

%\todo{AE: suggestion: mention advantages first, then mention possible disadvantage and its resolution.}
A second advantage of our approach is that, being an instance of Gaussian process
regression, there are well established estimates of predictive uncertainty \citep[Ch. 2]{RasWil06}.
We use these uncertainty estimates so as to determine when to query the importance sampler
for additional input/output pairs, i.e., the uncertain predictions trigger just-in-time updates
of the regressor. We demonstrate empirically that our uncertainty estimates are 
more robust and informative
than those for random forests, especially as we move away from the training data.

\begin{figure}[t]
\centering
\includegraphics[width=0.9\columnwidth]{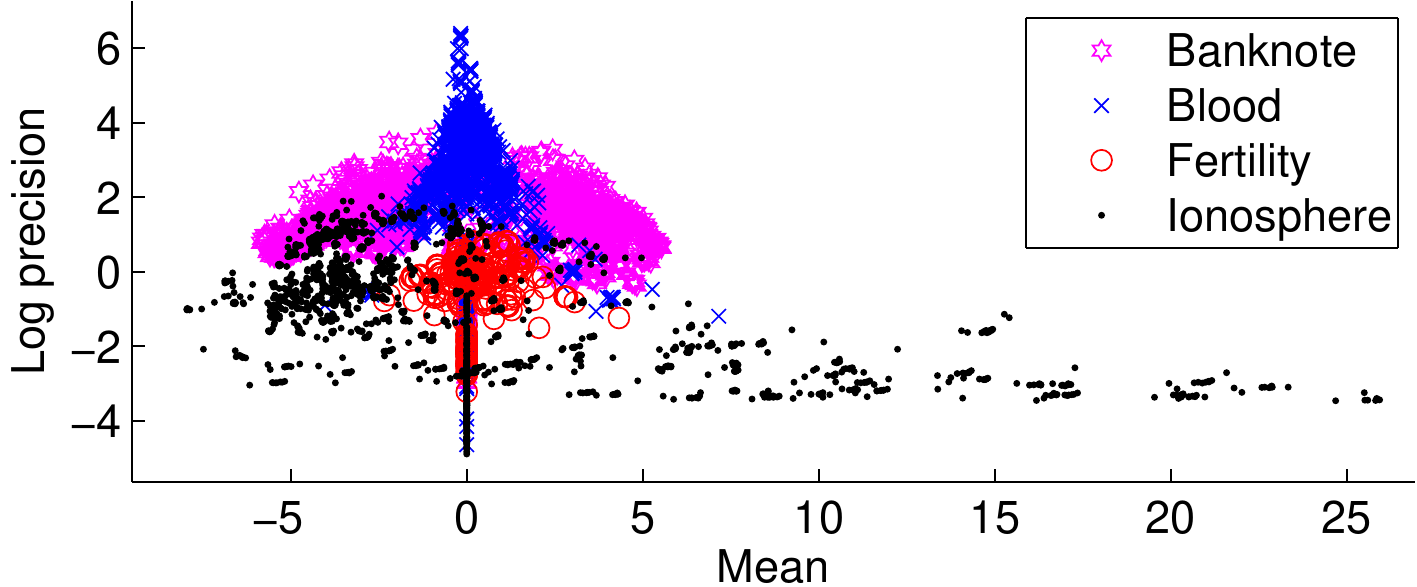}
\caption{Distributions of incoming messages to logistic factor in four different UCI datasets. 
%\wjnote{Refer to this figure in the intro. }
% AG: done
%\aenote{Figure fonts too large.}
\label{fig:uci_in_msgs}
}
\end{figure}

%Moreover, Song et al. did not employ the distribution regression of \cite{Szabo2014}  when learning the mapping from input messages
%to outgoing messages, as this  would not have been natural or easy to achieve in their setting
%(and indeed, was not an option available  at the time of this earlier work).

%\todo{AE: suggestion: to have a quick summary of contributions in one or two sentences somewhere?}
%AG: I think we are ok, and the introduction is already too long.
Our paper proceeds as follows. In \secref{sec:EP}, we introduce the notation for expectation propagation,
and indicate how an importance sampling procedure can be used as an oracle to provide training data for the message operator.
We also give a brief overview of previous learning approaches to the problem, with a focus on
that of \citet{Eslami2014}.
Next, in \secref{sec:Online}, we describe our kernel regression approach, and
the form of an efficient kernel message operator mapping the input messages (distributions embedded in an
RKHS) to outgoing messages (sets of parameters of the outgoing messages).
Finally, in \secref{sec:Experiments}, we describe our experiments, which cover three topics:
a benchmark of our uncertainty estimates, a demonstration of factor learning
on artificial data with well-controlled statistical properties,
and a logistic regression experiment on four different real-world datasets,
demonstrating that our just-in-time learner can correctly evaluate its uncertainty
and update the regression function as the incoming messages change (see \figref{fig:uci_in_msgs}).
Code to implement our method is available online at \url{https://github.com/wittawatj/kernel-ep}.

%AG: note code promise above :)

%%%%%%%%%%%%%%%%%%%%%%%%%%%%%%%%%%%%%%%%%%%%%%%%%%%%%%%%%%%%%%%%%%%%%%%%%%%%%%%%%%%%%%%%%%

\section{BACKGROUND}
\label{sec:EP}

We assume that distributions (or densities) $p$ over a set of variables 
$\bx = (x_1, \dots x_d)$ of interest can be represented as factor graphs, i.e.\
%
%\begin{equation*}
$p(\bx) = \frac{1}{Z} \prod_{j=1}^J \factor_j(\fx{\factor_j})$.
%\end{equation*}
The factors $\factor_j$ are non-negative functions which are defined over subsets $\fx{\factor_j}$ of the full set of variables $\bx$. These variables form the neighbors of the factor node $\factor_j$ in the factor graph, and we use $\fis{\factor_j}$ to denote the corresponding set of indices. $Z$ is the normalization constant.

We deal with models in which some of the factors have a non-standard form, or may not have a known analytic expression (i.e.\ ``black box'' factors). Although our approach applies to any such factor in principle, in this paper we focus on \textit{directed} factors $\factor(\xout | \xin)$ which specify a conditional distribution over variables $\xout$ given $\xin$ (and thus $\fx{\factor} = (\xout, \xin))$. The only assumption we make is that we are provided with a forward sampling function $f: \xin \mapsto \xout$, i.e., a function that maps (stochastically or deterministically) a setting of the input variables $\xin$ to a sample from the conditional distribution over $\xout \sim \factor(\cdot| \xin)$.
In particular, the ability to evaluate the value of $\factor(\xout | \xin)$ is not assumed.
A natural way to specify $f$ is as  code in a probabilistic program.

%--------------------------------------------------------------------

\subsection{EXPECTATION PROPAGATION}
\label{sec:EP:MP}

%AG: I like this explanation,but we are short of space, so I am commenting it out.
%Belief propagation -- or the sum-product algorithm -- computes marginal distributions over subsets of variables by iteratively passing messages between variables and factors, ensuring consistency of the obtained marginals at convergence. Specifically, the messages sent from a factor $\factor$ to variable $x_i$ (where $i \in \fis{\factor}$) are computed as
%\begin{equation}
%\msg{ \factor }{i}(x_i) = 
%\int \factor (\fx{\factor}) \prod_{i' \in \fis{\factor} \textbackslash i} \msg{i'}{\factor}(x_{i'}) \diffd \bx_{\fis{\factor} \textbackslash i},
%\label{eq:msgPassing:BP}
%\end{equation}
%where $\msg{i'}{\factor}$ are the messages sent to factor $\factor$ from its neighboring variables $x_{i'}$ other than $x_i$.

Expectation Propagation (EP) is an approximate iterative procedure for computing marginal beliefs of variables
by iteratively passing messages between variables and factors until convergence \citep{Minka2001}.
It can be seen as an alternative to belief propagation, where the marginals are projected
onto a member of some class of known parametric distributions. 
The message $\msg{ \factor }{\outV}(x_{\outV})$  from factor $\factor$ to variable $\outV\in\fis{\factor}$ is 
\begin{equation}
% \frac{ \projP{ \msg{\outV}{\factor}(x_{\outV})
% \int \factor (\fx{\factor}) \prod_{\outV'} \msg{\outV'}{\factor}(x_{\outV'}) \diffd \bx_{\fis{\factor} \textbackslash \outV}} }
% {\msg{\outV}{\factor}(x_{\outV})}
\frac{ \projP{ 
\int \factor (\fx{\factor}) \prod_{\outV' \in \fis{\factor}} \msg{\outV'}{\factor}(x_{\outV'}) \diffd 
\bx_{\fis{\factor} \backslash \outV}} }
{\msg{\outV}{\factor}(x_{\outV})},
\label{eq:msgPassing:EP}
\end{equation}
where $\msg{\outV'}{\factor}$ are the messages sent to factor $\factor$ from all of its neighboring variables $x_{\outV'}$,
%other than $x_{\outV}$,  
$\projP{p} = \argmin{q \in \mathcal{Q}} \kld{p}{q}$, and $\mathcal{Q}$ is typically in the exponential family, e.g.\ the set of Gaussian or Beta distributions.

%--------------------------------------------------------------------

%\subsection{MONTE-CARLO MESSAGE APPROXIMATION}
%\label{sec:EP:MC}

%By projecting messages onto simple parametric forms, EP introduces an approximation that allows message passing to proceed when the true messages are not easily representable in closed form.
Computing the numerator of (\ref{eq:msgPassing:EP}) can be challenging, as it requires evaluating a high-dimensional integral as well as minimization of the Kullback-Leibler divergence to some non-standard distribution. Even for factors with known analytic form this often requires hand-crafted approximations, or the use of expensive numerical integration techniques; for ``black-box'' factors implemented as forward sampling functions, fully nonparametric techniques are needed. 

%For such factors the only assumption we make is that we are provided with a forward sampling function $f: \xin \mapsto \xout$, i.e. a function that maps (stochastically or deterministically) a setting of the input variables $\xin$ to a sample from the conditional distribution over $\xout \sim \factor(\cdot| \xin)$. Such a function may, for instance, be

%This significantly limits the ease with which EP can be applied and its scope in practice. 
%However, while deterministic approximations to (\ref{eq:msgPassing:EP}) may be hard to come by,
\citet{Barthelme2011,Heess2013,Eslami2014} propose an alternative, stochastic approach to the integration and projection step.
When the projection  is to a member $q(x|\eta)=h(x)\exp\left(\eta^{\top}u(x)-A(\eta)\right)$ of an exponential family, one simply computes the expectation of the sufficient statistic $u(\cdot)$ under the numerator of (\ref{eq:msgPassing:EP}).
% \wjnote{Under the exact marginal $\msg{ \outV }{\factor}(x_{\outV})\msg{ \factor }{\outV}(x_{\outV})$  ?}
A sample based approximation of this expectation can be obtained via Monte Carlo simulation.   %AG: importance sampling is not MCMC
Given a forward-sampling function $f$ as described above, one especially simple approach is importance sampling, 
%\wjnote{$N$ is used for training size. Use $M$ here instead ?}
%AG: fixed
\begin{align}
\expectationE{u(x_{\outV})}{\fx{\factor}\sim b }
&\approx \frac{1}{M} \sum_{l=1}^M w(\fx{\factor}^l) u(x_{\outV}^l),
\label{eq:msgIS}
\end{align}
where $\fx{\factor}^l \sim \tilde{b}$, for $l=1,\ldots,M$ and on the left hand side, 
\begin{equation*}
b(\fx{\factor}) = \factor (\fx{\factor}) \prod_{W \in \fis{\factor}} \msg{W}{\factor}(x_{W}). 
\end{equation*}
On the right hand side we draw samples $\fx{\factor}^l$ from some proposal distribution $\tilde{b}$ which we choose to be 
%
% \begin{equation*}
$\tilde{b}(\fx{\factor}) = r(\xin)\factor(\xout | \xin)$
% \end{equation*}
%
for some distribution $r$ with appropriate support, and compute importance weights 
\begin{equation*}
w(\fx{\factor}) = \frac{\prod_{W \in \fis{\factor}} \msg{W}{\factor}(x_{W})}{r(\xin)}.
\end{equation*}

Thus the estimated expected sufficient statistics provide us with an estimate of the parameters $\eta$ of the result $q$ of the projection $\projP{p}$, from which the message is readily computed.\dsnote{isn't $\eta$ the actual message?}

%--------------------------------------------------------------------

\subsection{JUST-IN-TIME LEARNING OF MESSAGES}
\label{sec:EP:JIT}

Message approximations as in the previous section could be 
used directly when running the EP algorithm, as in \cite{Barthelme2011}, 
but this approach can suffer
when the number of samples $M$ is small,
and the  importance sampling estimate is not reliable. 
%AG: fixed
%\wjnote{Use $M$ ?}
On the other hand, for large $M$ the computational cost of running EP 
with approximate messages can be very high, 
%as the message approximations have to be re-computed in each iteration of EP. 
as importance sampling must be performed for sending each outgoing message.
To obtain low-variance message approximations at lower computational cost, \cite{Heess2013} and \cite{Eslami2014} both amortize previously computed approximate messages by training a function approximator to directly map a tuple of incoming variable-to-factor messages $(\msg{\outV'}{\factor} )_{\outV' \in \fis{\factor}}$ to an approximate factor to variable message $\msg{\factor}{\outV}$, i.e.\, they learn a mapping
\begin{equation}
\approxMsg{\factor}{\outV}{\theta}: (\msg{\outV'}{\factor} )_{\outV' \in \fis{\factor}} \mapsto \msg{\factor}{\outV},
\end{equation}
where $\theta$ are the parameters of the approximator.

%AG: shortened.
%Note that for exponential family distribution this effectively reduces to a multi-variate regression problem as each message can be represented by a finite-dimensional parameter vector, and the message update (\ref{eq:msgPassing:EP}) can thus be understood as a vector valued function.
%\wjnote{Why do we need to mention ``can be understood as a vector valued function'' ?}

\cite{Heess2013} use neural networks and a large, fixed training set to learn their approximate message operator prior to running EP. By contrast, \cite{Eslami2014}
employ random forests as their class of learning functions, and update their approximate message operator on the fly during inference, depending on the predictive uncertainty of the current message operator. Specifically, they endow their function approximator with an uncertainty estimate
\begin{equation}
\uncertaintyMsg{\factor}{\outV}{\theta}: (\msg{\outV'}{\factor} )_{\outV' \in \fis{\factor}} \mapsto \uncert,
\end{equation}
where $\uncert$ indicates the expected unreliability  of the predicted, approximate message $\msg{\factor}{\outV}$ returned by $\approxMsg{\factor}{\outV}{\theta}$. If $\uncert = \uncertaintyMsg{\factor}{\outV}{\theta} \left( (\msg{\outV'}{\factor} )_{\outV' \in \fis{\factor}}\right)$ 
exceeds a pre-defined threshold, the required message is approximated via
importance sampling (cf.\ \eqref{eq:msgIS}) and
$\approxMsg{\factor}{\outV}{\theta}$ is updated
%AG: fixed
%\wjnote{Just updated not retrained ?}
on this new datapoint (leading to a new set of parameters  $\theta'$ with $\uncertaintyMsg{\factor}{\outV}{\theta'} \left( (\msg{\outV'}{\factor} )_{\outV' \in \fis{\factor}}\right)) < \uncertaintyMsg{\factor}{\outV}{\theta} \left( (\msg{\outV'}{\factor} )_{\outV' \in \fis{\factor}}\right)$.

%The  performance of random forests in online learning relies crucially on having a reliable uncertainty estimate,
%and being able to robustly update the trees online without an explosion in computational cost and memory.
\cite{Eslami2014} estimate the predictive uncertainty  $\uncertaintyMsg{\factor}{\outV}{\theta}$  via
the heuristic of looking at the variability of the forest predictions for each point \citep{CriSho13}.
They implement their online updates by splitting the trees at their
leaves.
Both these mechanisms can be problematic, however. First, the heuristic
used in computing uncertainty has no guarantees: indeed, uncertainty estimation for
random forests remains a challenging topic of current research \citep{Hutter2009}. This is not merely a theoretical
consideration: in our experiments in \secref{sec:Experiments}, we demonstrate that 
uncertainty heuristics for random forests become unstable and inaccurate as we move away from the initial
training data. Second, online updates of random forests may not work well
when the newly observed data are from a very different distribution to the
initial training sample \citep[e.g.][Fig. 3]{LakRoyTeh14}. 
For large amounts of training set drift, the
leaf-splitting approach of \citeauthor{Eslami2014} can result in a decision tree in the form of a long chain, giving a worst case
cost of prediction (computational and storage) of $O(N)$ for training data of size $N$, vs the ideal of $O(\log(N))$
for balanced trees.
Finally, note that the approach of \citeauthor{Eslami2014}\ uses certain bespoke features of the factors when
specifying tree traversal in the random forests,
notably the value of the factor potentials at the mean and mode of the incoming messages.
These features 
require expert knowledge of the model on the part
of the practitioner, and are not available in the ``forward sampling'' setting. The present
work does not employ such features.

%and binary features indicating whether the inputs were proper, uniform, or point masses.
%\wjnote{The binary features are still accessible in the forward sampling setting.}
%While these features improved performance, they required expert knowledge of the model on the part
%of the practitioner, and would not be accessible in the ``forward sampling'' setting. The present
%work does not employ such features.

In terms of computational cost, prediction for the random forest of \citeauthor{Eslami2014} costs $O(K D_r D_t \log(N))$,
and updating following a new observation costs  $O(K D_r^3 D_t \log(N))$,
where $K$ is the number of trees in the random forest, $D_t$
is the number of features used in tree traversal, $D_r$ is the number of features used in making predictions
at the leaves, and $N$ is the number of training messages.
Representative values are $K=64$, $D_t=D_r\approx 15$, and $N$ in the range of  1,000 to 5,000.

\section{KERNEL LEARNING OF OPERATORS}\label{sec:Online}
%=============================================

We now propose a kernel regression method for jointly learning the message operator $\approxMsg{\factor}{\outV}{\theta}$ and
uncertainty estimate $\uncertaintyMsg{\factor}{\outV}{\theta}$. We regress from the tuple of incoming messages, which
are probability distributions, to the parameters of the outgoing message. 
To this end we apply a kernel over distributions from \citep{Christmann2010} 
to the case where the input consists of more than one distribution.

We note that \citet{SonGreGue10,SonGreBicLowGue11} propose a related regression approach for predicting outgoing
messages from incoming messages, for the purpose of belief propagation. 
Their setting is  different from ours, however, as their messages are smoothed conditional density functions rather than parametric distributions of known form.

To achieve fast predictions and factor updates, we follow
\citet{Rahimi2007,Le2013,YanSmoZonWil14}, and express the kernel regression in
terms of random features whose expected inner product
is equal to the kernel function; i.e.\ we perform regression directly in the primal on these random features.
In \secref{sec:kernelsOnDistributions}, we define our kernel on tuples of distributions, and then derive the corresponding random feature representation in 
\secref{sec:randomFeatureApproximations}. \secref{sec:ridgeRegression} describes the regression algorithm, as well as our strategy for uncertainty evaluation and online updates.

\subsection{KERNELS ON TUPLES OF DISTRIBUTIONS}\label{sec:kernelsOnDistributions}
%------------------------------------

In the following, we consider only a single factor, and therefore drop the factor identity from our notation. We write the set of $c$ incoming messages to a factor node as a tuple of
probability distributions %$\mathsf{x}_r:=$
$R:=(r^{(l)})_{l=1}^c$ of random variables $X^{(l)}$ on respective  domains $\mathcal{X}^{(l)}$.
Our goal is to define a kernel between one such tuple, and a second one,
which we will write %$\mathsf{x}_s:=$
$S:=(s^{(l)})_{l=1}^c$.

%AG: I have now simplified down to just the kernel we use.

%The simplest kernel between distributions is simply a standard kernel (e.g., a Gaussian kernel) on a parametric representation of the inputs.
%Thus,  if each input distribution $r^{(l)}$ is represented by a mean-variance pair $(\eta_{r^{(l)}},\sigma^2_{r^{(l)}})$, then a kernel
%between the message tuples is a product of kernels $\kappa^{(l)}(r^{(l)},s^{(l)})$ between each pair of distributions,
%where the $\kappa^{(l)}$ are themselves products of a Gaussian kernels between the means, and a Gaussian kernel between the variances. We call this the {\em mean-variance kernel}.

We define our kernel in terms of embeddings of the tuples $R,S$ into a reproducing kernel Hilbert space (RKHS). We first consider the embedding of a single distribution
in the tuple: Let us define an RKHS
$\mathcal{H}^{(l)}$ on each domain, with respective kernel $k^{(l)}(x^{(l)}_1,x^{(l)}_2)$.
We may embed individual probability distributions to these RKHSs, following \citep{Smola2007}.
The {\em mean embedding} of $r^{(l)}$ is written
\begin{equation}
\mu_{r^{(l)}}(\cdot) := \int k^{(l)}(x^{(l)},\cdot) \, dr^{(l)} (x^{(l)}).
\end{equation}
Similarly, a mean embedding may be defined on the product of  messages in a tuple
$\mathsf{r}=\times_{l=1}^{c}r^{(l)}$ as
\begin{equation}
\mu_{\mathsf{r}}
:=
%\int k([x^{(1)} \hdots x^{(c)}],\cdot) \prod_{l=1}^{c}dr^{(l)}(x^{(l)}),
\int k([x^{(1)}, \ldots, x^{(c)}],\cdot) \, d\mathsf{r}(x^{(1)}, \ldots, x^{(c)}),
\label{eq:featureOfTuple}
\end{equation}
where we have defined the joint kernel $k$ on the product space
%concatenation of the features
$\mathcal{X}^{(1)}\times\cdots\times\mathcal{X}^{(c)}$.
Finally,  a kernel on two such embeddings $\mu_{\mathsf{r}},\mu_{\mathsf{s}}$ of tuples $R,S$ can be obtained as in \citet[eq. 9]{Christmann2010},
\begin{equation}
\kappa(\mathsf{r}, \mathsf{s}) = \exp\left(-\frac{\|\mu_{\mathsf{r}}-\mu_{\mathsf{s}}\|_{\mathcal{H}}^{2}}{2\gamma^{2}}\right).
\label{eq:gauss_joint_emb}
\end{equation}
This kernel has two parameters: $\gamma^{2}$, and the width parameter of the
kernel $k$ defining $\mu_{\mathsf{r}} = \mathbb{E}_{x \sim \mathsf{r}}k(x,\cdot)$.

We have considered several alternative kernels on tuples of messages, including
kernels on the message parameters, kernels on a tensor feature space of the
distribution embeddings in the tuple, and dot products of the features
(\ref{eq:featureOfTuple}). We have found these alternatives to have worse
empirical performance than the approach described above. We give details of
these experiments in \secref{sec:batch_learning_detail} of the
supplementary material.

\subsection{RANDOM FEATURE APPROXIMATIONS}
%------------------------------------
\label{sec:randomFeatureApproximations}
%AG: replaced ``kernel ridge'' wiht Gaussian process'' so as not to switch back and forth.

One approach to learning the mapping  $\approxMsg{\factor}{\outV}{\theta}$ from incoming to outgoing messages
would be to employ Gaussian process regression, using the kernel \eqref{eq:gauss_joint_emb}.
This approach is not suited to just-in-time (JIT) learning, however,
as both prediction and storage costs grow with the size of the training set;
thus, inference on even moderately sized datasets rapidly becomes computationally prohibitive.
Instead, we define 
a finite-dimensional random feature map $\hat{\psi} \in \mathbb{R}^{D_\mathrm{out}}$ such that 
$\kappa(\mathsf{r}, \mathsf{s}) \approx \hat{\psi}(\mathsf{r})^\top \hat{\psi}(\mathsf{s})$, 
and regress directly on these feature maps in the primal (see next section): storage and computation 
are then a function of the dimension of the feature map $D_\mathrm{out}$, yet performance is close to that
obtained using a kernel.

%% Our message operator $\approxMsg{\factor}{\outV}{\theta}$ is based on kernel ridge 
%% regression using the kernel defined in \eqref{eq:gauss_joint_emb}. 
%% However, kernel ridge regression in its dual form is not suitable for JIT learning 
%% as its solution is expressed in terms of inner products between the messages in 
%% the training set. This is problematic for our JIT learning 
%% purpose because the training size grows as new examples 
%% (incoming-outgoing message pairs) arrive. 
%% One possible solution to this is to approximate  the kernel with 
%% a finite-dimensional random feature map $\hat{\phi} \in \mathbb{R}^{D_\mathrm{out}}$ such that 
%% $\kappa(\mathsf{r}, \mathsf{s}) \approx \hat{\psi}(\mathsf{r})^\top \hat{\psi}(\mathsf{s})$, 
%% which will then allow us 
%% to do kernel ridge regression in its primal form: equivalently this is just the 
%% standard linear regression with the basis expansion function $\hat{\psi}$. 

% The kernel $\kappa_{\text{gauss}}(\mathsf{r}, \mathsf{s})$ defines an inner product 
% $\kappa(\mathsf{r}, \mathsf{s}) = \langle 
% \phi(\mathsf{r}), \phi(\mathsf{s}) \rangle_\mathcal{H}$ between two 
% elements $\phi(\mathsf{r})$ and $\phi(\mathsf{s})$ in a Hilbert space $\mathcal{H}$.
% When $\mathcal{H}$ is infinite-dimensional (as is the case for the space induced by 
% kernel in \eqref{eq:gauss_joint_emb}), an explicit feature map $\phi$ 
% cannot be computed.

In \cite{Rahimi2007}, a method based on Fourier transforms was proposed for computing a vector
of random features $\hat{\varphi}$ for a translation invariant kernel $k(x,y) = k(x-y)$
% the standard Gaussian kernel 
% $k(x, y) = \exp\left(-\frac{\| x-y\|^2}{2\beta^2} \right)$, 
 such that $k(x, y) \approx \hat{\varphi}(x)^\top \hat{\varphi}(y)$
where $x,y \in \mathbb{R}^d$ and $\hat{\varphi}(x), \hat{\varphi}(y) \in \mathbb{R}^{D_\mathrm{in}}$.  
This is possible because of Bochner's theorem \citep{Rudin2013}, which states that a
continuous, translation-invariant kernel $k$ can be written in the form of an
inverse Fourier transform:
\begin{equation*}
k(x-y)=\int\hat{k}(\omega)e^{j\omega^{\top}\left(x-y\right)}\,
d\omega,
\end{equation*}
where $j=\sqrt{-1}$ and the Fourier transform $\hat{k}$ of the kernel can be
treated as a distribution. The inverse Fourier transform can thus be seen as
an expectation of the complex exponential, which can be approximated with a
Monte Carlo average by drawing random frequencies from the Fourier transform. 
We will follow a similar approach, and derive 
a two-stage set of random Fourier features for \eqref{eq:gauss_joint_emb}.
%via a two-stage 
%approximations.

We start by expanding the  exponent 
of \eqref{eq:gauss_joint_emb} as
\begin{align*}
 %& \kappa_{\text{gauss}}(p,q) \\
  \exp\left(-\frac{1}{2\gamma^{2}}\left\langle \mu_{\mathsf{r}},\mu_{\mathsf{r}}\right\rangle +\frac{1}{\gamma^{2}}\left\langle \mu_{\mathsf{r}},\mu_{\mathsf{s}}\right\rangle -\frac{1}{2\gamma^{2}}\left\langle \mu_{\mathsf{s}},\mu_{\mathsf{s}}\right\rangle \right).
\end{align*}
Assume that the embedding kernel $k$ used to define the embeddings $\mu_\mathsf{r}$ 
and $\mu_\mathsf{s}$ is translation invariant. Since 
$\langle \mu_{\mathsf{r}},\mu_{\mathsf{s}}  \rangle
= \mathbb{E}_{x \sim \mathsf{r}} \mathbb{E}_{y \sim \mathsf{s}} k(x-y)$, one can use 
the result of \cite{Rahimi2007} to write
\begin{align*}
 \langle \mu_{\mathsf{r}},\mu_{\mathsf{s}}  \rangle
 & \approx \mathbb{E}_{x \sim \mathsf{r}} \mathbb{E}_{y \sim \mathsf{s}} 
   \hat{\varphi}(x)^\top \hat{\varphi}(y) \nonumber \\ 
 & = \mathbb{E}_{x \sim \mathsf{r}} 
   \hat{\varphi}(x)^\top \mathbb{E}_{y \sim \mathsf{s}}  \hat{\varphi}(y) 
 := \hat{\phi}(\mathsf{r})^\top \hat{\phi}(\mathsf{s}),
\end{align*}
where the mappings $\hat{\phi}$ are $D_\mathrm{in}$ standard Rahimi-Recht random features, shown in Steps 1-3 of Algorithm~\ref{algo:random_features_kgg}.

With the approximation of $\langle \mu_{\mathsf{r}},\mu_{\mathsf{s}}  \rangle$,
we have
\begin{equation}
\kappa(\mathsf{r}, \mathsf{s})\approx\exp\left(-\frac{\|\hat{\phi}(\mathsf{r})-\hat{\phi}(\mathsf{s})\|_{D_\mathrm{in}}^{2}}{2\gamma^{2}}\right),
%
%:=k_{\text{gauss}}\left(\hat{\phi}(p)-\hat{\phi}(q);\gamma^{2}\right)
\end{equation}
which is a standard Gaussian kernel on $\mathbb{R}^{D_\mathrm{in}}$.
We can thus further approximate this Gaussian kernel 
%
% $ k_{\text{gauss}}\left(\hat{\phi}(p), \hat{\phi}(q);\gamma^{2}\right)\approx\hat{\psi}(p)^{\top}\hat{\psi}(q)$
%
by the random Fourier features of \citeauthor{Rahimi2007}, to obtain a vector 
of random features $\hat{\psi}$ such that 
$\kappa(\mathsf{r}, \mathsf{s}) \approx \hat{\psi}(\mathsf{r})^\top \hat{\psi}(\mathsf{s})$
where $\hat{\psi}(\mathsf{r}), \hat{\psi}(\mathsf{s}) \in \mathbb{R}^{D_\mathrm{out}}$. 
Pseudocode for generating the random features $\hat{\psi}$ is given in
Algorithm~\ref{algo:random_features_kgg}. 
Note that the sine component in the complex exponential vanishes due to the
translation invariance property (analogous to an even function), i.e.,  only 
the cosine term remains.
We refer to \secref{sub:Expected-Product-Kernel} in the supplementary material for more details.

For the implementation, we need to pre-compute $\left\{ \omega_{i}\right\} _{i=1}^{D_\mathrm{in}},\left\{ b_{i}\right\} _{i=1}^{D_\mathrm{in}},\left\{ \nu_{i}\right\} _{i=1}^{D_\mathrm{out}}$
and $\left\{ c_{i}\right\} _{i=1}^{D_\mathrm{out}}$, where $D_\mathrm{in}$ and
$D_\mathrm{out}$ are the number of random features used. 
A more efficient way to support a large number of random features 
is to store only the random seed used to generate 
the features, and to generate the coefficients  on-the-fly as needed \citep{Dai2014}. 
In our implementation, we use a Gaussian kernel for $k$.

%%%

\begin{algorithm}[t]
\caption{Construction of two-stage random features for $\kappa$}
\label{algo:random_features_kgg}
\begin{algorithmic}[1]
\REQUIRE Input distribution $\mathsf{r}$, Fourier transform $\hat{k}$ of 
the embedding translation-invariant kernel $k$, number of inner features $D_\mathrm{in}$, number of outer features $D_\mathrm{out}$, outer Gaussian width $\gamma^2$.
\ENSURE Random features $\hat{\psi}(\mathsf{r}) \in \mathbb{R}^{D_\mathrm{out}}$. 

%\STATE Compute the Fourier transform $\hat{k}$ of the kernel $k$.
\STATE Sample  $\{ \omega_i \}_{i=1}^{D_\mathrm{in}} \overset{i.i.d}{\sim} \hat{k}$.
\STATE Sample $\{b_i\}_{i=1}^{D_\mathrm{in}} \overset{i.i.d}{\sim} \text{Uniform}[0, 2\pi] $.
\STATE $\hat{\phi}(\mathsf{r}) = \sqrt{\frac{2}{D_\mathrm{in}}} \left( \mathbb{E}_{x \sim \mathsf{r}} 
\cos(\omega_{i}^{\top}x+b_{i} ) \right)_{i=1}^{D_\mathrm{in}} \in \mathbb{R}^{D_\mathrm{in}}$ \\
%\STATE  $\hat{\phi}(p)=\mathbb{E}_{p(x)}\sqrt{\frac{2}{D_\mathrm{in}}}\left(\cos\left(\omega_{1}^{\top}x+b_{1}\right),\ldots,\cos\left(\omega_{D_\mathrm{in}}^{\top}x+b_{D_\mathrm{in}}\right)\right)^{\top}$.
If $\mathsf{r}(x)=\mathcal{N}(x;m, \Sigma )$, 
\small
\begin{equation*}
\hat{\phi}( \mathsf{r}) = \sqrt{\frac{2}{D_\mathrm{in}}} \left( \cos(\omega_{i}^{\top}m +b_{i}) \exp 
\left(-\frac{1}{2}\omega_{i}^{\top}\Sigma \omega_{i} \right) \right)_{i=1}^{D_\mathrm{in}}.
\end{equation*}
%Even if $p$ is not a normal distribution, we may still use it as an approximation.
%
\STATE Sample $\{ \nu_i \}_{i=1}^{D_\mathrm{out}} \overset{i.i.d}{\sim} \hat{k}_{\text{gauss}}(\gamma^{2})$
i.e., Fourier transform of a Gaussian kernel with width $\gamma^2$.
\STATE Sample $\{c_i\}_{i=1}^{D_\mathrm{out}} \overset{i.i.d}{\sim} \text{Uniform}[0, 2\pi] $.
\STATE $\hat{\psi}(\mathsf{r}) = \sqrt{\frac{2}{D_\mathrm{out}}} \left(  
\cos(\nu_{i}^{\top} \hat{\phi}(\mathsf{r}) + c_{i} ) \right)_{i=1}^{D_\mathrm{out}} \in 
\mathbb{R}^{D_\mathrm{out}}$
\end{algorithmic}
\end{algorithm}

\subsection{REGRESSION FOR OPERATOR PREDICTION}\label{sec:ridgeRegression}

Let $\mathsf{X}=\left(\mathsf{x}_{1}|\cdots|\mathsf{x}_{N}\right)$
be the $N$ training samples of incoming messages to a factor node, and let
$\mathsf{Y}=\left(\mathbb{E}_{x_V \sim q_{\factor\rightarrow
\outV}^{1}}u(x_{\outV})|\cdots|\mathbb{E}_{x_V \sim q_{f\rightarrow
\outV}^{N}}u(x_{\outV})\right)\in\mathbb{R}^{D_{y}\times N}$
be the expected sufficient statistics of the corresponding  output messages, 
where $q^i_{\factor \rightarrow \outV}$ is the numerator 
of \eqref{eq:msgPassing:EP}.
We write $\mathsf{x}_{i}= \hat{\psi}(\mathsf{r}_i)$
as a more compact notation for the random feature
vector representing the $i^{th}$ training tuple of incoming messages,
as computed via Algorithm~\ref{algo:random_features_kgg}. 
%Thus, the dimensionality of  $\mathsf{X}$ is $D_{out} \times N$.

%\agnote{Add the sans serif notation below.}
%% As mentioned, kernel ridge regression in the primal form is equivalent to  
%% standard ridge regression, which is equivalent to Bayesian linear regression 
%% by interpreting the regularization parameter as the output noise variance.
%% A probabilistic interpretation is useful as we seek an uncertainty estimate 
%% on test points.

Since we require uncertainty estimates on our predictions,
we perform Bayesian linear regression from the random features to the output messages,
which yields predictions close to those obtained by Gaussian process regression
with the kernel in \eqref{eq:gauss_joint_emb}.
The uncertainty estimate in this case will be the predictive 
variance.
%Here we give a brief review of Bayesian linear regression on the random features. 
We assume prior and likelihood
\begin{align}
w & \sim\mathcal{N}\left(w;0,I_{D_\mathrm{out}}\sigma_{0}^{2}\right), \\
\mathsf{Y} \mid \mathsf{X},w & \sim\mathcal{N}\left(\mathsf{Y};w^{\top} \mathsf{X},\sigma_{y}^{2}I_{N}\right),
\end{align}
where the output noise variance $\sigma_{y}^{2}$ captures the intrinsic
stochasticity of the importance sampler used to generate $\mathsf{Y}$. It
follows that the posterior of $w$ is given by \citep{Bishop2006}
\begin{align}
p(w | \mathsf{Y}) & =\mathcal{N}(w;\mu_{w},\Sigma_{w}), \\
\Sigma_{w} & = \left( \mathsf{X} \mathsf{X}^{\top}\sigma_{y}^{-2}+\sigma_{0}^{-2}I \right)^{-1}, \\
\mu_{w} & =\Sigma_{w} \mathsf{X} \mathsf{Y}^{\top}\sigma_{y}^{-2}.
%=\left(XX^{\top}+\frac{\sigma_{y}^{2}}{\sigma_{0}^{2}}I\right)^{-1}XY^{\top}.
\end{align}
%The noise variance $\sigma_{y}^{2}$ is proportional to the regularization
%parameter in linear regression. 
The predictive distribution on the output $\mathsf{y}^{*}$ given an 
observation $\mathsf{x}^{*}$ is
\begin{align}
p(\mathsf{y}^{*}| \mathsf{x}^{*}, \mathsf{Y}) & =\int  
 p(\mathsf{y}^{*}|w, \mathsf{x}^{*}, \mathsf{Y}) p(w|\mathsf{Y}) \, dw\\
 & =\mathcal{N}\left(\mathsf{y}^{*}; \mathsf{x}^{*\top}\mu_{w}, \mathsf{x}^{*\top}\Sigma_{w} \mathsf{x}^{*}+\sigma_{y}^{2}\right).
% & =\mathcal{N}(y^{*};m^{*},v^{*}).
\end{align}
%
% The marginal likelihood is given by
%\begin{align*}
%p(Y|X) & =\int dw\, p(Y|X,w)p(w)\\
% & =\mathcal{\mathcal{N}}\left(Y;0,K\sigma_{0}^{2}+\sigma_{y}^{2}I_{N}\right)\\
% & =\det\left(2\pi\Sigma\right)^{-1/2}\exp\left(-\frac{1}{2}Y^{\top}\left(K\sigma_{0}^{2}+\sigma_{y}^{2}I_{N}\right)^{-1}Y\right)
%\end{align*}
%where $K=X^{\top}X$ is a Gram matrix of $X$. 
%One can do gradient ascent on $\log p(Y|X)$
%to optimize $\sigma_{0}^{2},\sigma_{y}^{2}$ and kernel parameters.
%
For simplicity, we treat each output (expected sufficient statistic) as a separate regression problem. 
Treating all outputs jointly can be achieved with a multi-output kernel \citep{Alvarez2011}.

%AG: don't need a subsection
%\subsection{Online Update \label{sec:kernel_ep_online}}
 
\paragraph{Online Update}
We describe an online update for $\Sigma_{w}$ and
$\mu_{w}$ when observations (i.e., random features representing incoming 
messages) $\mathsf{x}_i$ arrive sequentially. We use $\cdot^{(N)}$
to denote a quantity constructed from $N$ samples. Recall that $\Sigma_{w}^{-1(N)}= \mathsf{X} \mathsf{X}^{\top}\sigma_{y}^{-2}+\sigma_{0}^{-2}I$.
The posterior covariance matrix at time $N+1$ is
\begin{equation}
\Sigma_{w}^{(N+1)} 
 =
\Sigma_{w}^{(N)}-\frac{\Sigma_{w}^{(N)} \mathsf{x}_{N+1} \mathsf{x}_{N+1}^{\top} \Sigma_{w}^{(N)}\sigma_{y}^{-2}}{1+ \mathsf{x}_{N+1}^{\top}\Sigma_{w}^{(N)} \mathsf{x}_{N+1}\sigma_{y}^{-2}},
\end{equation}
%AG: don't need this proof

%\begin{align*}
%\Sigma_{w}^{(N+1)} & =\left(XX^{\top}\sigma_{y}^{-2}+ \sigma_{0}^{-2}I  
% + x_{N+1}x_{N+1}^{\top}\sigma_{y}^{-2} \right)^{-1}\\
% & =\left(\Sigma_{w}^{-1(N)}+x_{N+1}x_{N+1}^{\top}\sigma_{y}^{-2}\right)^{-1}\\
% & =\Sigma_{w}^{(N)}-\frac{\Sigma_{w}^{(N)}x_{N+1}x_{N+1}^{\top}\Sigma_{w}^{(N)}\sigma_{y}^{-2}}{1+x_{N+1}^{\top}\Sigma_{w}^{(N)}x_{N+1}\sigma_{y}^{-2}}
%\end{align*}
%where we used the Sherman-Morrison formula. 
%In this form, the posterior covariance at time $N+1$
meaning it can be expressed
as an inexpensive update of
the covariance at time $N$.
%\[
%\left(A+uv^{\top}\right)^{-1}=A^{-1}-\frac{A^{-1}uv^{\top}A^{-1}}{1+v^{\top}A^{-1}u}
%\]
%with $A=\Sigma_{w}^{-1(N)}$. 
Updating $\Sigma_{w}$ for all the $D_y$ outputs costs 
$O( (D_\mathrm{in}D_\mathrm{out} + D_\mathrm{out}^{2}) D_y)$ 
per new observation. 
For $\mu_{w}= \Sigma_{w} \mathsf{X} \mathsf{Y}^{\top}\sigma_{y}^{-2}$, we maintain
$ \mathsf{X} \mathsf{Y}^{\top}\in\mathbb{R}^{D_{\mathrm{out}}\times D_{\mathrm{y}}}$, and update it
at cost $O(D_\mathrm{in}D_\mathrm{out}D_y)$ as
\begin{equation}
\left( \feaX \feaY^{\top}\right)^{(N+1)}=\left( \feaX \feaY^{\top}+ \feax_{N+1} \feay^\top_{N+1}\right).
\end{equation}
%In the first iteration when no data are observed, $\Sigma_{w}^{(1)}=\sigma_{0}^{2}I$.
%
%If the predictive variance of the current incoming messages $x^{*}$
%is high, the operator queries the correct outgoing message from the importance sampler
%(oracle) and update the $\Sigma_{w}$ and $\mu_{w}$. 
%Otherwise, the outgoing message is efficiently computed by the operator. 
%
Since we have $D_{y}$ regression functions, 
for each tuple of incoming messages $\feax^{*}$, there are $D_{y}$
predictive variances, $v_{1}^{*},\ldots,v_{D_{y}}^{*}$, one for each
output. 
Let $\{\tau_{i}\}_{i=1}^{D_{y}}$ be pre-specified predictive variance thresholds.
%\wjnote{How should we choose these thresholds ?}
Given a new input $\feax^{*}$, if $v_{1}^{*}>\tau_{1}$ or $\cdots$
or $v_{D_{y}}^{*}>\tau_{D_{y}}$ (the operator is uncertain), 
a query is made to the oracle to obtain a ground truth $\feay^{*}$. 
The pair $(\feax^{*}, \feay^{*})$ is then
used to update $\Sigma_{w}$ and $\mu_{w}$.% as described previously. 
%Querying the oracle incurs a high computational cost.

%%%%%%%%%%%%%%

%=============================================
\section{EXPERIMENTS  }
\label{sec:Experiments}
%=============================================

We evaluate our learned message operator using two different factors: the logistic factor, and the compound gamma factor. In the first and second experiment we demonstrate that the proposed operator is capable of learning high-quality mappings from incoming to outgoing messages, and that the associated uncertainty estimates are  reliable. The third and fourth experiments assess the performance of the operator as part of the full EP inference loop in two different models: approximating the logistic, and the compound gamma factor. Our final experiment demonstrates the  ability of our learning process to reliably and quickly adapt to large shifts in the message distribution, as encountered during inference in a sequence of several real-world regression problems.

For all experiments we used Infer.NET \citep{Minka2014} with its extensible factor interface for our own operator. 
We used the default settings of Infer.NET unless stated otherwise. 
The regression target is the marginal belief (numerator of \eqref{eq:msgPassing:EP}) in experiment 1,2,3 and 5. We set the regression target to the outgoing message in experiment 4. 
Given a marginal belief, the outgoing message can be calculated straightforwardly.
% as
%$\msg{\factor}{x} = q_{\factor \rightarrow x} / \msg{x}{\factor} $.
%$q_{\factor \rightarrow z_i}$
%(i.e., the numerator of \eqref{eq:msgPassing:EP}) rather than 
%the outgoing message $\msg{\factor}{z_i}$. 
%An 

%------------------------------------
%\subsection{Logistic factor \label{sec:logistic_factor} } 
%------------------------------------

\begin{figure}[ht]
\centering
%\missingfigure{Factor graph for binary logistic regression. Just imagine by yourself for now.}
%\includegraphics[scale=0.8,page=2,clip,trim=8cm 18.5cm 8cm 1cm]{img/heess_passing_ep_supp.pdf}
%\includegraphics[scale=0.8,page=1,clip,trim=6.5cm 7.5cm 7cm 17.8cm]{img/eslami_jit_supp.pdf}
\includegraphics[width=0.8\columnwidth]{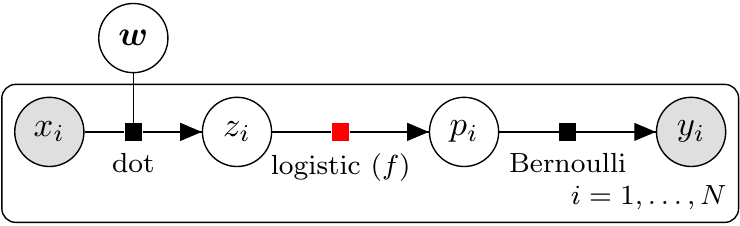}
% \begin{tikzpicture}
%  \node[obs] (x) {$x_i$};
%  \bayesfactor[right= of x] {dot} {below:dot} {} {};
%  \node[latent, above = 5mm of dot] (w) {$\boldsymbol{w}$};
%  \node[latent, right = 6mm of dot] (z) {$z_i$};
%  \bayesfactor[right= 6mm of z, color=red] {logistic} {below:logistic} {} {};
%  \node[latent, right = 6mm of logistic] (p) {$p_i$};
%  \bayesfactor[right = 6mm of p] {bern} {below:Bernoulli} {} {};
%  \node[obs, right = 6mm of bern]  (y)   {$y_i$}; %
%  
%  \edge[-] {dot} {x} ;
%  \edge[-] {w} {dot};
%  \edge[-] {dot} {z} ;
%  \edge[-] {z} {logistic} ;
%  \edge[-] {logistic} {p};
%  \edge[-] {p} {y} ;
%  
%   \plate {sample} { %
%     (x)  (z) (p) (y)
%   } {$i=1, \ldots, N$} ;  
% \end{tikzpicture}
\caption{Factor graph for binary logistic regression. 
The kernel-based message operator learns to approximate the logistic factor 
highlighted in red. The two incoming messages are 
$\msg{z_i}{\factor} = \mathcal{N}(z_i; \mu, \sigma^2)$ and 
$\msg{p_i}{\factor} = \text{Beta}(p_i; \alpha, \beta) $. 
}
\label{fig:factor_graph_binlog}
\end{figure}

\paragraph{Experiment 1: Batch Learning} 
As in \citep{Heess2013,Eslami2014}, we study the logistic factor 
$\factor(p|z)=\delta\left(p-\frac{1}{1+\exp(-z)}\right),$ 
%
% \begin{equation*}
% \factor(p|z)=\delta\left(p-\frac{1}{1+\exp(-z)}\right)
% \end{equation*}
%
where $\delta$ is the Dirac delta function, in the context of 
a binary logistic regression model  (\figref{fig:factor_graph_binlog}).
The factor is deterministic and there are two incoming messages: 
$\msg{p_i}{\factor} = \text{Beta}(p_i; \alpha, \beta) $ and 
$\msg{z_i}{\factor} = \mathcal{N}(z_i; \mu, \sigma^2)$, 
where $z_i = \boldsymbol{w}^\top x_i$ represents the dot product between an observation 
$x_i \in \mathbb{R}^d$ and the coefficient vector $\boldsymbol{w}$ whose posterior is 
to be inferred.

In this first experiment we simply learn a kernel-based operator to send the message $\msg{\factor}{z_i}$.
Following \cite{Eslami2014}, we set $d$ to 20, and generated 20 different datasets, sampling 
a different $\boldsymbol{w} \sim \mathcal{N}(0, I)$ and then a set of $\{(x_i, y_i)\}_{i=1}^n$ ($n=300$) observations according to the model.
For each dataset we ran EP for 10 iterations, and collected incoming-outgoing message pairs in 
the first five iterations of EP from Infer.NET's implementation of the 
logistic factor.
We partitioned the messages randomly into 5,000 training and 
3,000 test messages, and learned a message operator to predict $m_{f\rightarrow z_i}$
as described in \secref{sec:Online}. 
Regularization and kernel parameters were chosen by leave-one-out cross validation.
%(note that cross-validation can be performed 
%faster using a reduced  number of random features, albeit
%with some increase in noise).
We set the number of random features to $D_{in}=500$ and $D_{out}=1,000$; 
empirically, we observed no significant improvements beyond 1,000 random features.
%Since an EP outgoing message can be improper e.g., a normal distribution with 
%negative variance, the regression target was set to be the belief message 
%$q_{\factor \rightarrow z_i}$
%(i.e., the numerator of \eqref{eq:msgPassing:EP}) rather than 
%the outgoing message $\msg{\factor}{z_i}$. 
%An outgoing message can be constructed straightforwardly as
%$\msg{\factor}{z_i} = q_{\factor \rightarrow z_i} / \msg{z_i}{\factor} $.
%=======
%All collected messages were randomly partitioned into 5,000 training messages and 
%3,000 testing messages.
%We learned a message operator to predict $m_{f\rightarrow X}$
%using a Gaussian kernel on joint embeddings, as 
%described in \eqref{eq:gauss_joint_emb}.
%Regularization and kernel parameters were chosen by leave-one-out cross validation
%(note that cross-validation can be performed 
%faster using a reduced  number of random features, albeit
%with some increase in noise).
%The number of random features was set to $D_\mathrm{in}=500$ and $D_\mathrm{out}=1,000$.
%Empirically, we observed no significant improvements beyond 1,000 random features.
%Since an EP outgoing message can be improper e.g., a normal distribution with 
%negative variance, the regression target was set to be the belief message 
%$q_{\factor \rightarrow z_i}$
%(i.e., the numerator of \eqref{eq:msgPassing:EP}) rather than 
%the outgoing message $\msg{\factor}{z_i}$. 
%An outgoing message can be constructed straightforwardly as
%$\msg{\factor}{z_i} = q_{\factor \rightarrow z_i} / \msg{z_i}{\factor} $.
%>>>>>>> 1a716b28202c0f2b7e105940d649e83eada16f1d

\begin{figure}[ht]
  \centering

  \subfloat[KL errors \label{fig:kl_div_hist}]{
  \includegraphics[width=0.46\columnwidth]{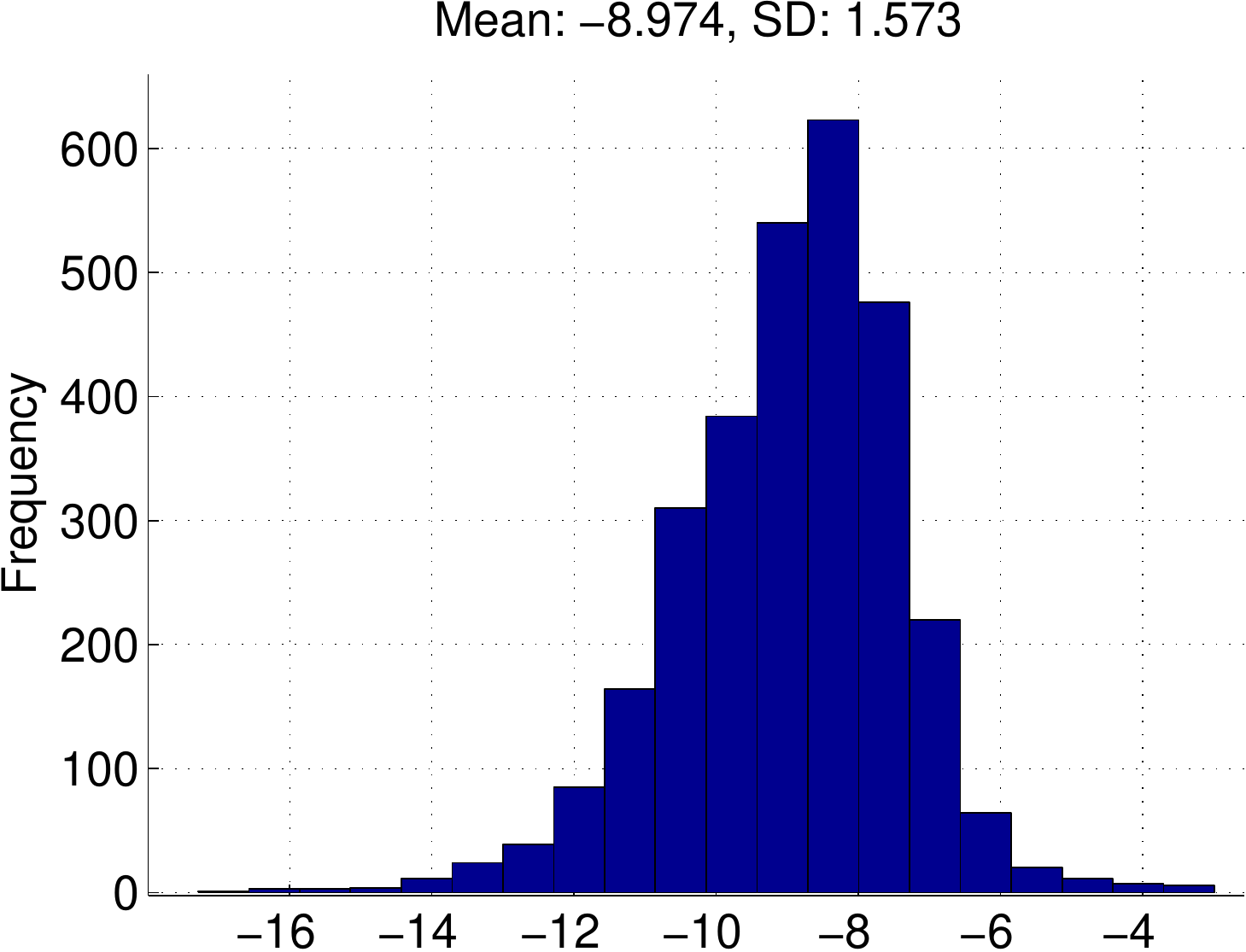}
  }
  \subfloat[Examples of predictions\label{fig:kl_div_4plots}]{
  \includegraphics[width=0.47\columnwidth]{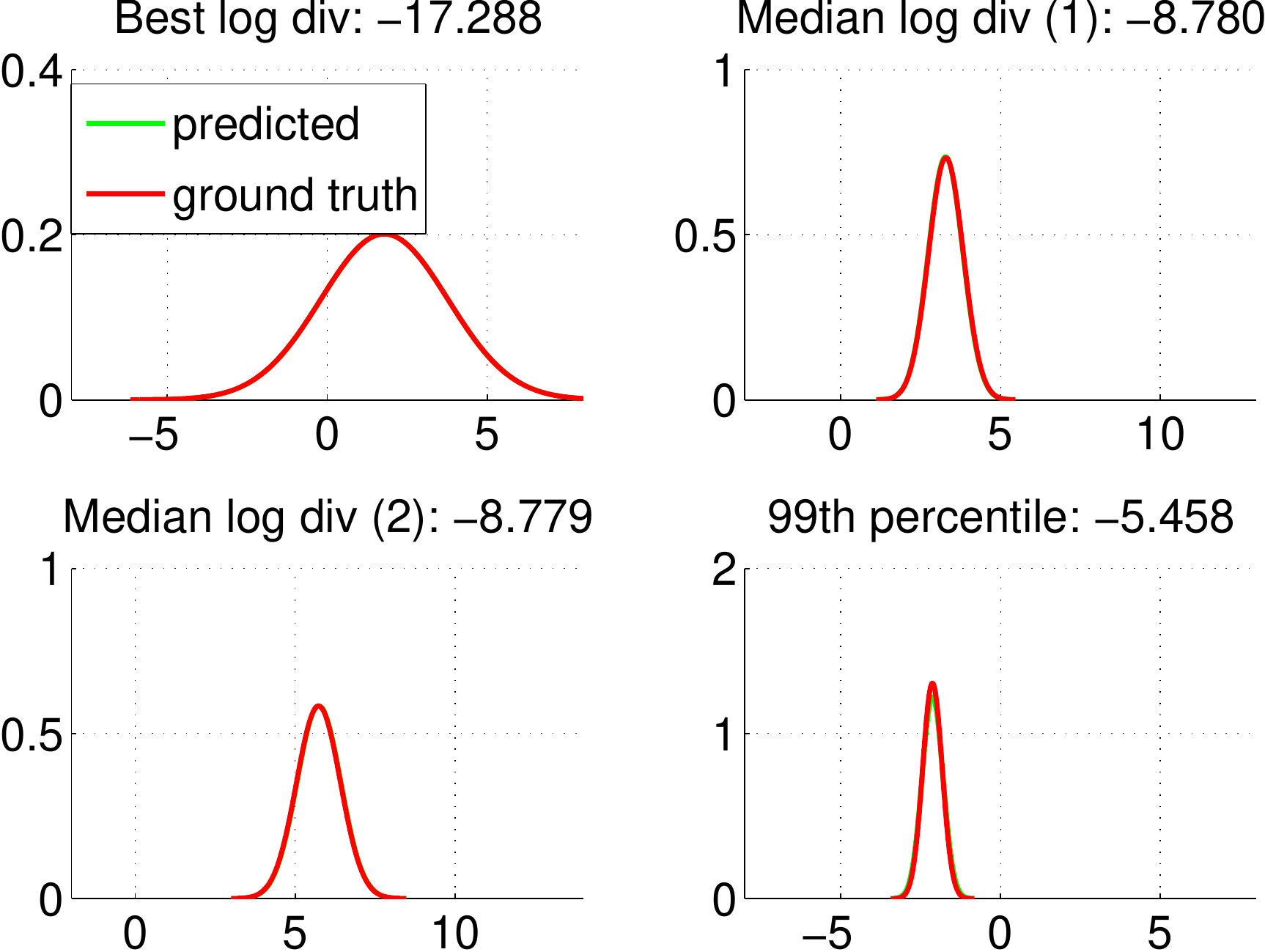}
  }
  \caption{Prediction errors for predicting the projected beliefs to $z_i$, and examples of predicted messages at different error levels. 
  }
  % reference: exp7/RFGJointKGGLearner_binlogis_bw_proj_n400_iter5_sf1_st20_ntr5000.mat
  \label{fig:kl_div}
\end{figure}

We report 
$\log \mathrm{KL}[q_{\factor \rightarrow z_i} \| \hat{q}_{\factor \rightarrow z_i}]$ 
where $q_{\factor \rightarrow z_i}$ is the ground truth projected belief (numerator of  \eqref{eq:msgPassing:EP}) and 
$\hat{q}_{\factor \rightarrow z_i}$ is the prediction.
% For better numerical scaling, regression outputs are set to $(\mathbb{E}_{q}\left[x\right],
% \log\mathbb{V}_{q}\left[x\right])$ instead of the expectations of the first two moments. 
% \aenote{AE: No need to mention log-scaling?} 
The histogram of the log KL errors is shown in \figref{fig:kl_div_hist}; \figref{fig:kl_div_4plots} shows examples of predicted messages 
for different log KL errors. 
It is evident that the kernel-based operator does well in capturing the relationship
between incoming and outgoing messages. The discrepancy with respect to the ground truth is barely 
visible even at the 99th percentile. 
See \secref{sec:batch_learning_detail} in
the supplementary material for a comparison with other methods.

\paragraph{Experiment 2: Uncertainty Estimates}

\begin{figure}[t]
\centering
  \subfloat[Parameters of $\msg{z_i}{\factor}$ \label{fig:logistic_uncertainty_test}]{
  \includegraphics[width=0.49\columnwidth]{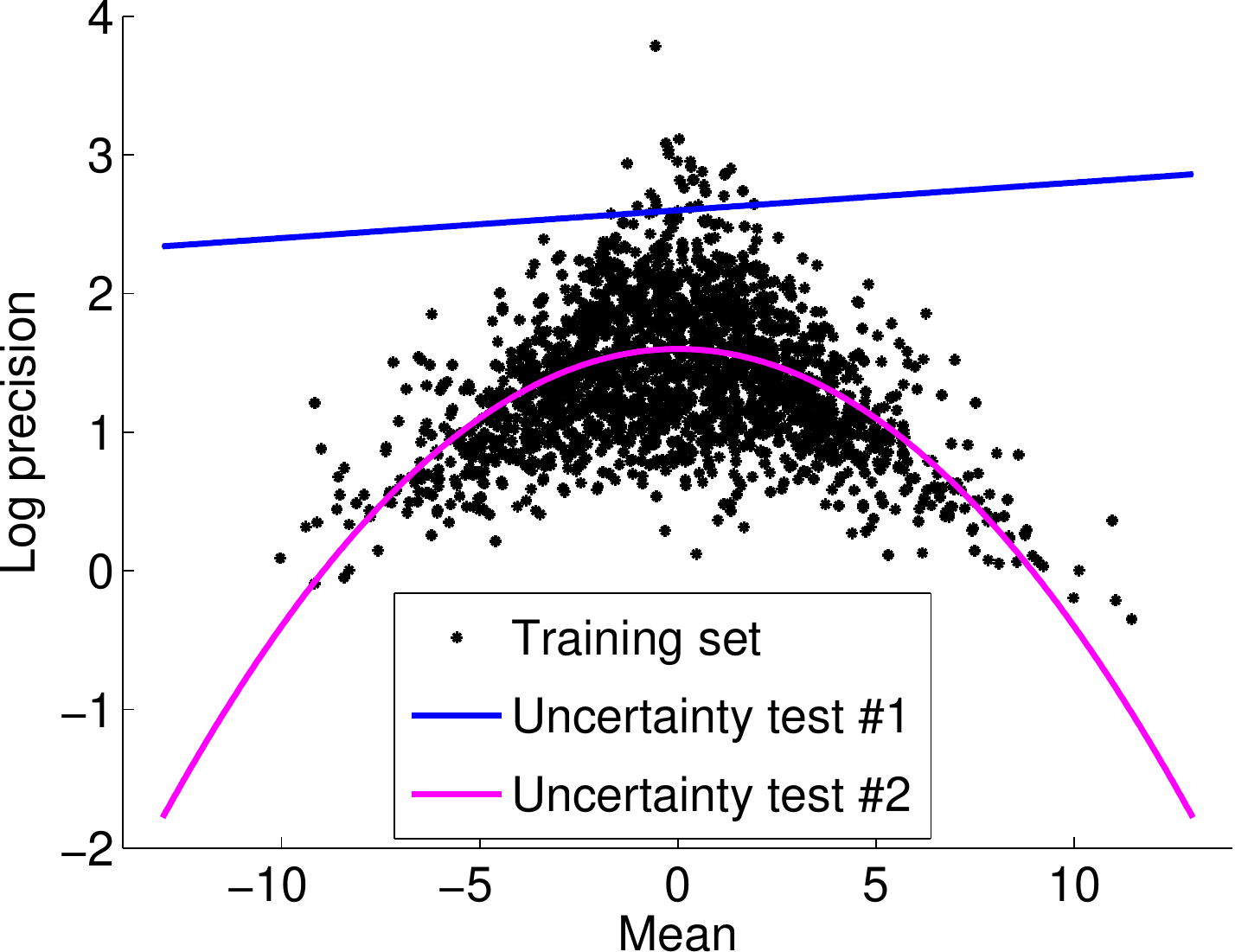}
  }
  %
%   \subfloat[Kernel-based operator \label{fig:logistic_predvar}]{
%   \includegraphics[width=0.3\textwidth]{img/uncertainty/logistic_predvar-crop}
%   }
  \subfloat[Uncertainty estimates \label{fig:logistic_uncertainty_all}]{
  \includegraphics[width=0.49 \columnwidth]{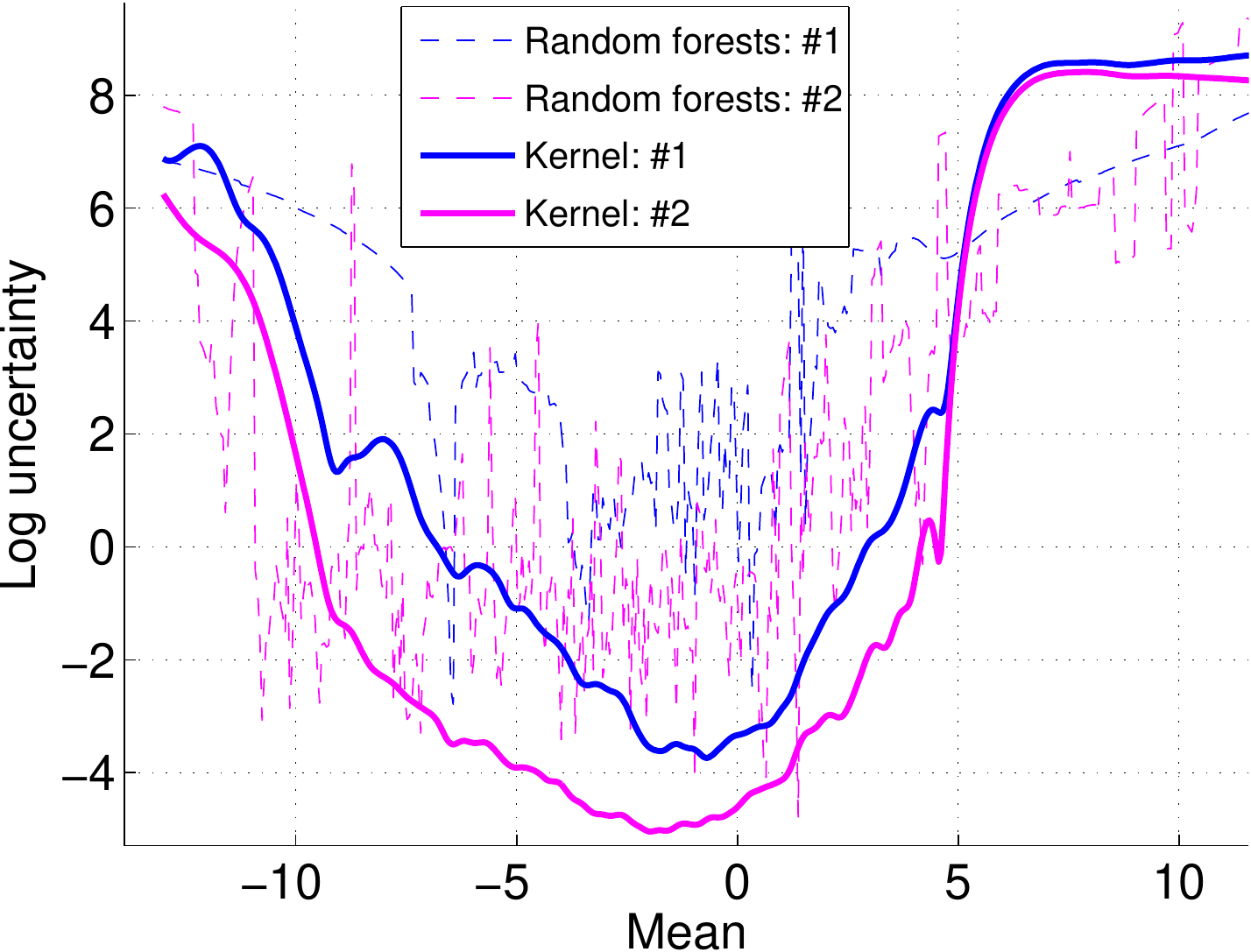}
  }
  %
%   \subfloat[Breiman's random forests\label{fig:breiman_uncertainty}]{
%   \includegraphics[width=0.49\columnwidth]{img/uncertainty/rf-n_trees-100-crop}
%   }
%   %
%   \subfloat[Extremely randomized trees\label{fig:ert_uncertainty}]{
%   \includegraphics[width=0.49\columnwidth]{img/uncertainty/ert-n_trees-100-crop}
% % (c,d) The uncertainty estimates on the same test sets of Breiman's random 
% %   forests and extremely randomized trees. 
%   }
%   \subfloat[\citeauthor{Eslami2014}'s random forests \label{fig:eslami_uncertainty}]{
%   \includegraphics[width=0.3\textwidth]{img/uncertainty/uncertainty-data-10dim-method-ali_rf-n_candidates-100-regressor-quadratic-n_trees-64-crop}
%   }
  \caption{ (a) Incoming messages from $z$ to $\factor$ from 20 EP runs of binary 
  logistic regression, as shown in \figref{fig:factor_graph_binlog}. 
  (b) Uncertainty estimates of the proposed kernel-based method (predictive variance) and 
  \citeauthor{Eslami2014}'s random forests (KL-based agreement of predictions of different trees) 
  on the two uncertainty test sets shown. For testing, we fix the other incoming message 
  $\msg{p_i}{\factor}$ to $\text{Beta}(p_i; 1, 2)$.
  }
  \label{fig:logistic_uncertainty}
%   \caption{Uncertainty estimates of random forests on two uncertainty test sets 
%   shown in \figref{fig:logistic_predvar_unexplored}. } 
\end{figure}

For the approximate message operator to perform well in a JIT learning setting, it is crucial to have reliable estimates 
of operator's predictive uncertainty in different parts of the space of incoming messages.
%in unexplored regions in the space of incoming messages. 
To assess this property  we compute the predictive variance using the same learned operator as used in 
\figref{fig:kl_div}. The forward incoming messages $\msg{z_i}{\factor}$ in the previously 
used training set are shown in \figref{fig:logistic_uncertainty_test}. 
The backward incoming messages $\msg{p_i}{\factor}$ are not displayed.
Shown in the same plot are two curves (a blue line, and a pink parabola) representing two ``uncertainty test sets'':
these are the sets of parameter pairs on which we wish to evaluate
the uncertainty of the predictor, and pass through regions
with both high and low densities of training samples.
\figref{fig:logistic_uncertainty_all}
% \figref{fig:breiman_uncertainty}, and \figref{fig:ert_uncertainty} 
% and \figref{fig:eslami_uncertainty} 
shows uncertainty estimates
of our
kernel-based operator and of
random forests,
where we fix 
$\msg{p_i}{\factor} := \text{Beta}(p_i; 1, 2)$ for testing. 
The implementation of the random forests closely
follows 
%that of 
\cite{Eslami2014}. 

From the figure, as the mean of the test message moves away from the region densely sampled by the training data, the predictive variance reported by the kernel method increases
much more smoothly than that of the random forests. Further, our method
clearly exhibits a higher uncertainty on the test set \#1 than on the test set \#2.
This behaviour is desirable, as most of the points in  test set \#1 are either 
in a low density region or an unexplored region. These results suggest that the 
predictive variance is a robust criterion for querying the importance sampling oracle.
One key observation is that the uncertainty estimates of the random 
forests are highly non-smooth; i.e., uncertainty on nearby points may vary wildly.
As a result, a random forest-based JIT learner may still query the importance 
sampler oracle when 
presented with incoming messages similar to those in the training set, 
thereby wasting computation.

We have further checked that the predictive uncertainty of the regression function is a
reliable indication of the error in KL divergence of the predicted outgoing messages. These results
are given in Figure \ref{fig:logistic_predvar_g} of Appendix \ref{sec:batch_learning_detail}.

%\begin{figure}
%\centering
%\includegraphics[width=0.49\columnwidth]{img/uncertainty/rf-n_trees-100-crop}
%\includegraphics[width=0.49\columnwidth]{img/uncertainty/ert-n_trees-100-crop}
%\caption{Log KL-divergence on a logistic factor test set using kernel on joint embeddings.}
%\label{fig:tree_uncertainty_unexplored}
%\end{figure}

\paragraph{Experiment 3: Just-In-Time Learning}
In this experiment we test the approximate operator in the logistic regression model as part of the full EP inference loop in a just-in-time learning setting (KJIT). % on the same task with toy data. 
We now learn two kernel-based message operators, one for each outgoing 
direction from the logistic factor. 
The data generation is the same as in the batch learning experiment.
We sequentially presented the operator with 30 related problems, where a new 
set of observations $\{(x_i, y_i)\}_{i=1}^n$ was generated at the beginning of 
each problem from the model, while keeping $\boldsymbol{w}$ fixed.
This scenario is common in practice: one is often given several sets of 
observations which share the same model parameter \citep{Eslami2014}. 
As before, the inference target was $p(\boldsymbol{w}|\{(x_i, y_i)\}_{i=1}^n)$.
We set the maximum number of EP iterations to 10 in each problem.

% logistic temporal uncertainty
\begin{figure*}[t]
\centering
\includegraphics[width=0.95\textwidth]{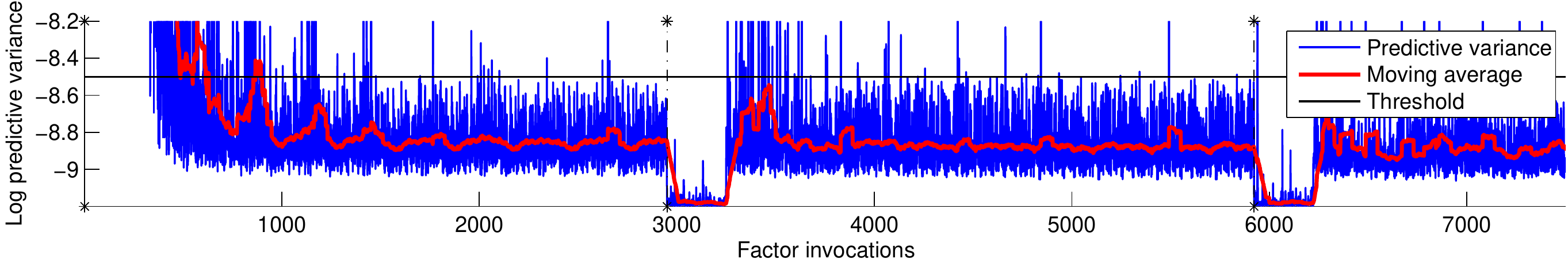}
\caption{Uncertainty estimate of KJIT in its prediction of outgoing messages at each factor invocation,
for the binary logistic regression problem. The black dashed lines indicate the start 
of a new inference problem.
\label{fig:logistic_temporal_uncertainty}
}
\end{figure*}

%\wjnote{The term KJIT should be introduced somewhere before here.}
%AG: I guess so. I've just introduced it in the paragraph above.

We employed a ``mini-batch'' learning approach in which the operator always consults the oracle in
the first few hundred factor invocations for initial batch training. 
In principle, during the initial batch training, the operator can perform 
cross validation or type-II maximum likelihood estimation for parameter
selection; however for computational simplicity
we set the kernel parameters  according to the median heuristic
\citep{Scholkopf2002}. Full detail of the heuristic is given in
\secref{sec:median_heuristic} in the supplementary material. 
%See \secref{sec:median_heuristic} in the supplementary
%material for more detail.
%AG: removed cross validation discussion from here, put it earlier.
The numbers 
of random features were $D_\mathrm{in} = 300$ and $D_\mathrm{out} = 500$. The output noise variance 
$\sigma^2_y$ was fixed to $10^{-4}$ and the uncertainty threshold on the log 
predictive variance was set to -8.5. To simulate a black-box setup, we used
an importance sampler as the oracle rather than Infer.NET's factor implementation, 
where the proposal distribution was fixed to $\mathcal{N}(z; 0, 200)$ with 
$5 \times 10^5$ particles.

\figref{fig:logistic_temporal_uncertainty} shows a trace of the predictive variance 
of KJIT in predicting the mean of each $\msg{\factor}{z_i}$ upon each factor invocation. 
The black dashed lines indicate the start of a new inference problem. 
Since the first 300 factor invocations are for the initial training, 
no uncertainty estimate is shown. From the trace, we observe that the uncertainty 
rapidly drops down to a stable point at roughly -8.8 and levels 
off after the operator sees about 1,000 incoming-outgoing message pairs, 
which is relatively low compared to approximately 3,000 message passings 
(i.e., 10 iterations $\times$ 300 observations) required for one problem. 
The uncertainty trace displays a periodic structure, repeating itself in 
every 300 factor invocations, corresponding to a full sweep over all 300 
observations to collect incoming messages $\msg{z_i}{\factor}$. 
The abrupt drop in uncertainty in the first EP iteration of each new problem is 
due to the fact that Infer.NET's inference engine initializes the message from 
$\boldsymbol{w}$ to have zero mean, leading to $\msg{z_i}{\factor}$ also having a 
zero mean. Repeated encounters of such a zero mean incoming message reinforce the 
operator's confidence; hence the drop in uncertainty. 

% logistic classification error. Inference time
\begin{figure}[t]
  \centering
  \subfloat[Binary classification error\label{fig:logistic_01_loss}]{
  \includegraphics[width=0.49\columnwidth]{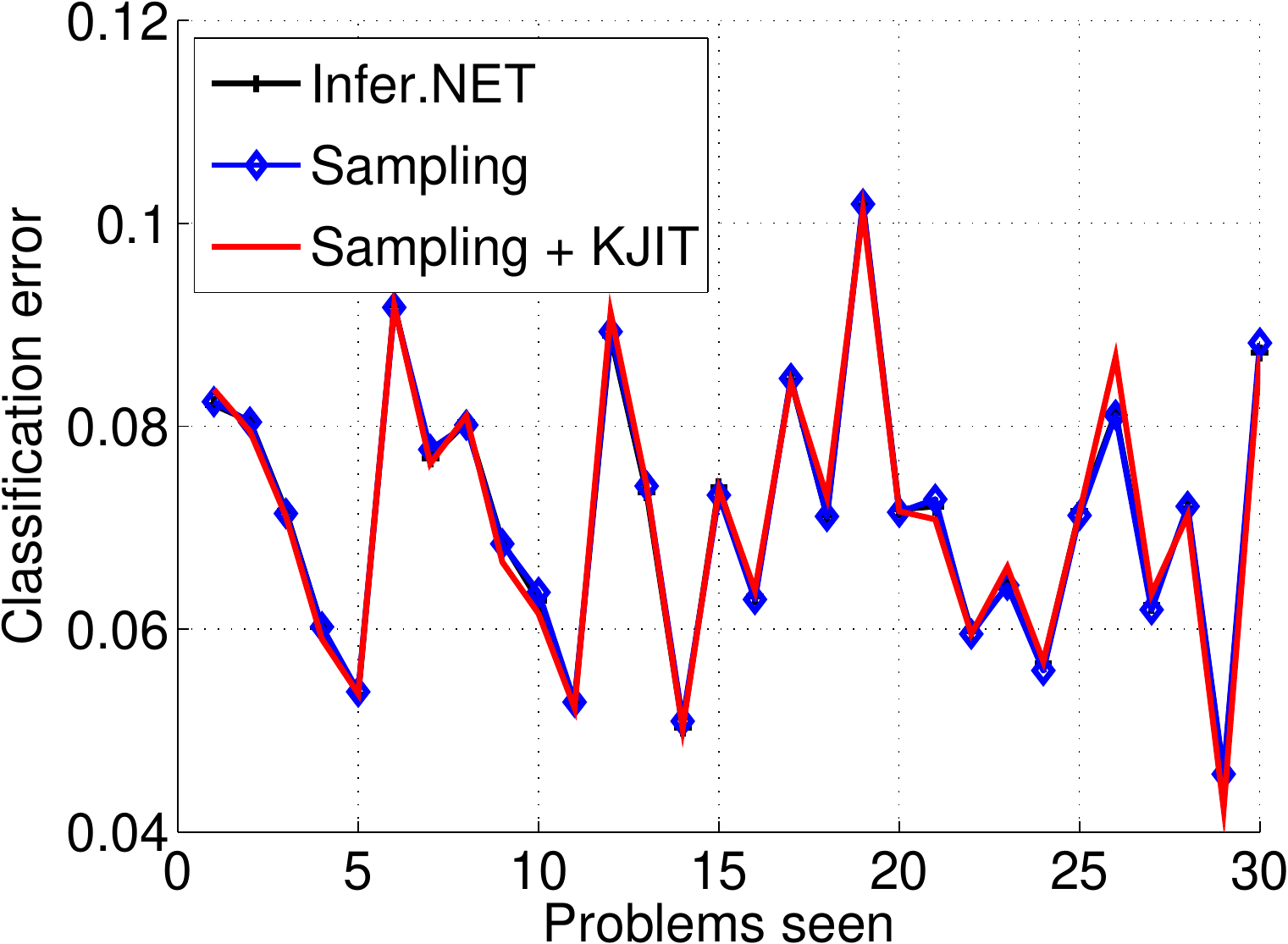}
  }
  \subfloat[Inference time\label{fig:logistic_inference_time}]{
  \includegraphics[width=0.47\columnwidth]{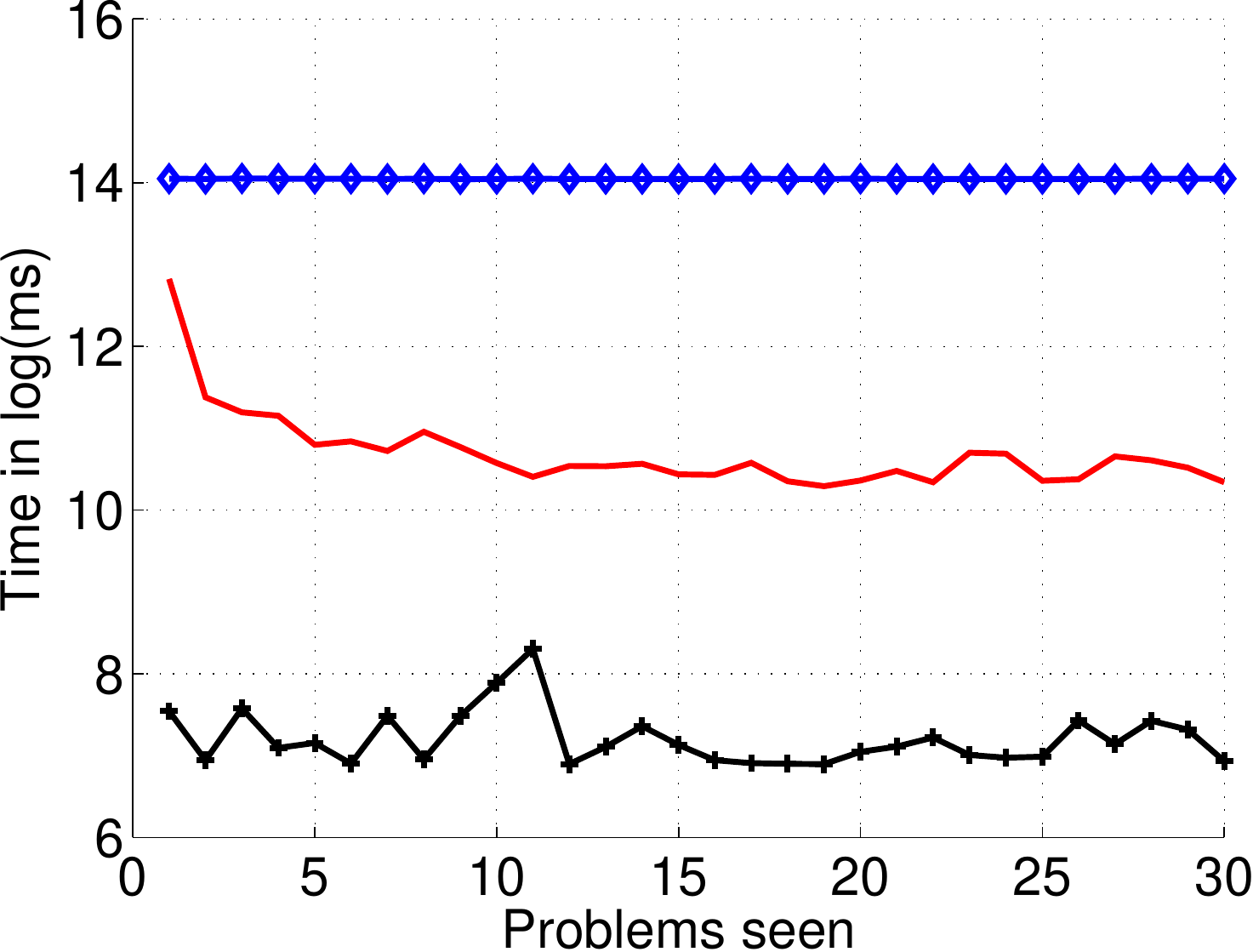}
  }
  \caption{Classification performance and inference times of all methods in the 
  binary logistic regression problem. 
  }
  \label{fig:logistic_performance}
\end{figure}

\figref{fig:logistic_01_loss} shows binary classification errors obtained by 
using the inferred posterior mean of $\boldsymbol{w}$ on a test set of size 
10000 generated from the true underlying parameter. 
Included in the plot are the errors obtained by using only the importance 
sampler for inference (``Sampling''), and using the Infer.NET's 
hand-crafted logistic factor. The loss of KJIT matches well with that of 
the importance sampler and Infer.NET, suggesting that the inference accuracy
is as good as these alternatives. \figref{fig:logistic_inference_time} 
shows the inference time required by all methods in each problem. 
While the inference quality is equally good, KJIT is orders of magnitude faster 
than the importance sampler.

\begin{comment}
%I referred to the generated C# code to see how the inference engine initialized.
% Check out ArrayHelper.MakeUniform(..)

  private void Changed_dataCount_Init_numberOfIterationsDecreased_X_Y(bool initialise)
  {
    if ((this.Changed_dataCount_Init_numberOfIterationsDecreased_X_Y_iterationsDone==1)&&((!initialise)||this.Changed_dataCount_Init_numberOfIterationsDecreased_X_Y_isInitialised)) {
	    return ;
    }
    for(int n = 0; n<this.DataCount; n++) {
	    this.W_rep_B[n] = ArrayHelper.MakeUniform<VectorGaussian>(VectorGaussian.Uniform(this.vVector0.Count));
    }
    this.Changed_dataCount_Init_numberOfIterationsDecreased_X_Y_iterationsDone = 1;
    this.Changed_dataCount_Init_numberOfIterationsDecreased_X_Y_isInitialised = true;
    this.Changed_numberOfIterationsDecreased_dataCount_X_Y_iterationsDone = 0;
  }

\end{comment}

%------------------------------------
%\subsection{Compound gamma factor}
%------------------------------------

% compound gamma temporal uncertainty
% \begin{figure*}[t]
% \centering
% \includegraphics[width=0.95\textwidth]{img/online/cg_temporal_uncertainty-crop}
% \caption{Uncertainty estimate of KJIT for all incoming messages at each time point 
% in the compound Gamma problem. 
% 
% %\aenote{Top of legend missing. Legend items could be `Predictive variance', `Moving average', `Oracle consultation', `Threshold'. Title should be predictive variance of the \textit{outgoing} message.}
% \label{fig:cg_temporal_uncertainty}
% }
% \end{figure*}

% compound gamma inferred results and time.
\begin{figure}[ht]
  \centering
% 	\subfloat[Posteriors]{
% 	\includegraphics[width=0.49\columnwidth]{img/online/cg_post_corr-crop}
% 	%\missingfigure[figwidth=0.49\columnwidth]{}
% 	}
  %
  \subfloat[Inferred shape \label{fig:cg_infer_shape}]{
  \includegraphics[width=0.33\columnwidth]{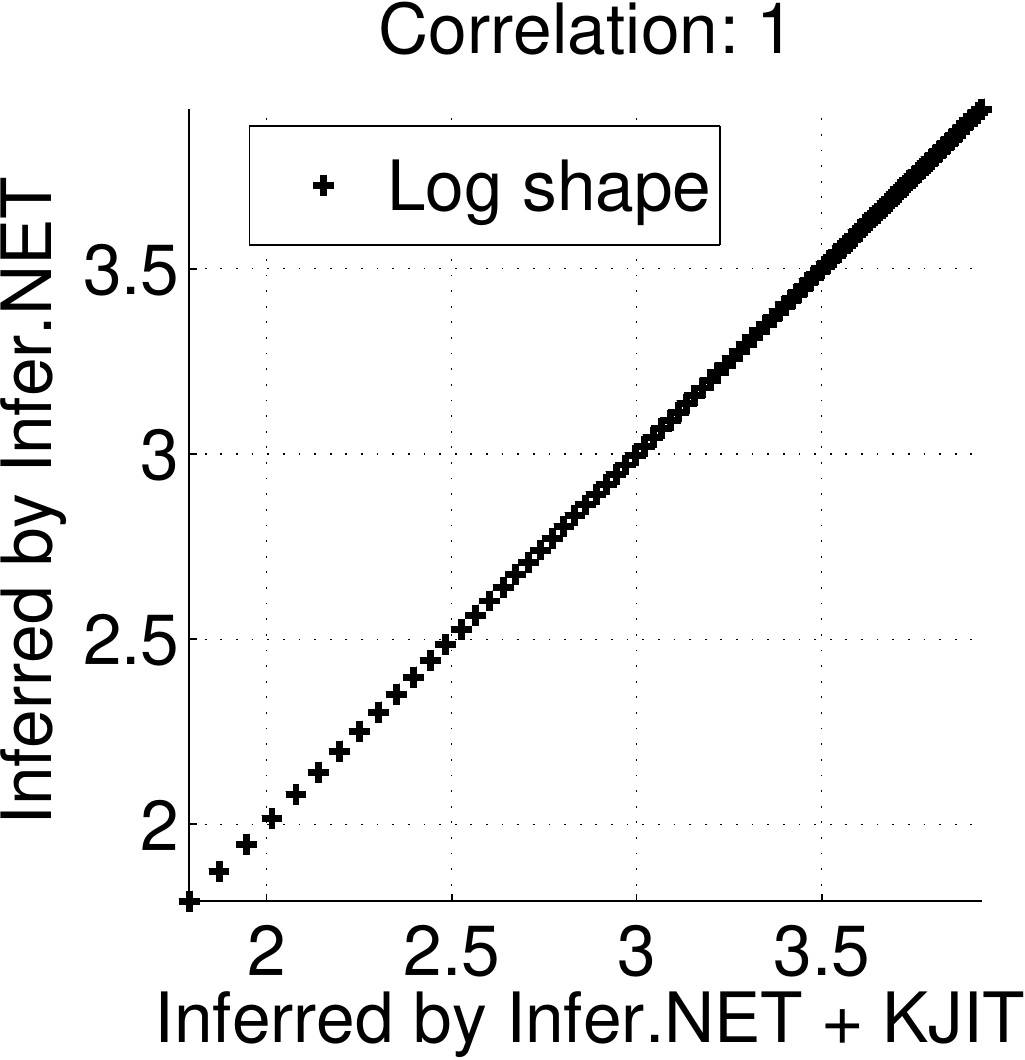}
  %\missingfigure[figwidth=0.49\columnwidth]{}
  }
  \subfloat[Inferred rate \label{fig:cg_infer_rate}]{
  \includegraphics[width=0.31\columnwidth ]{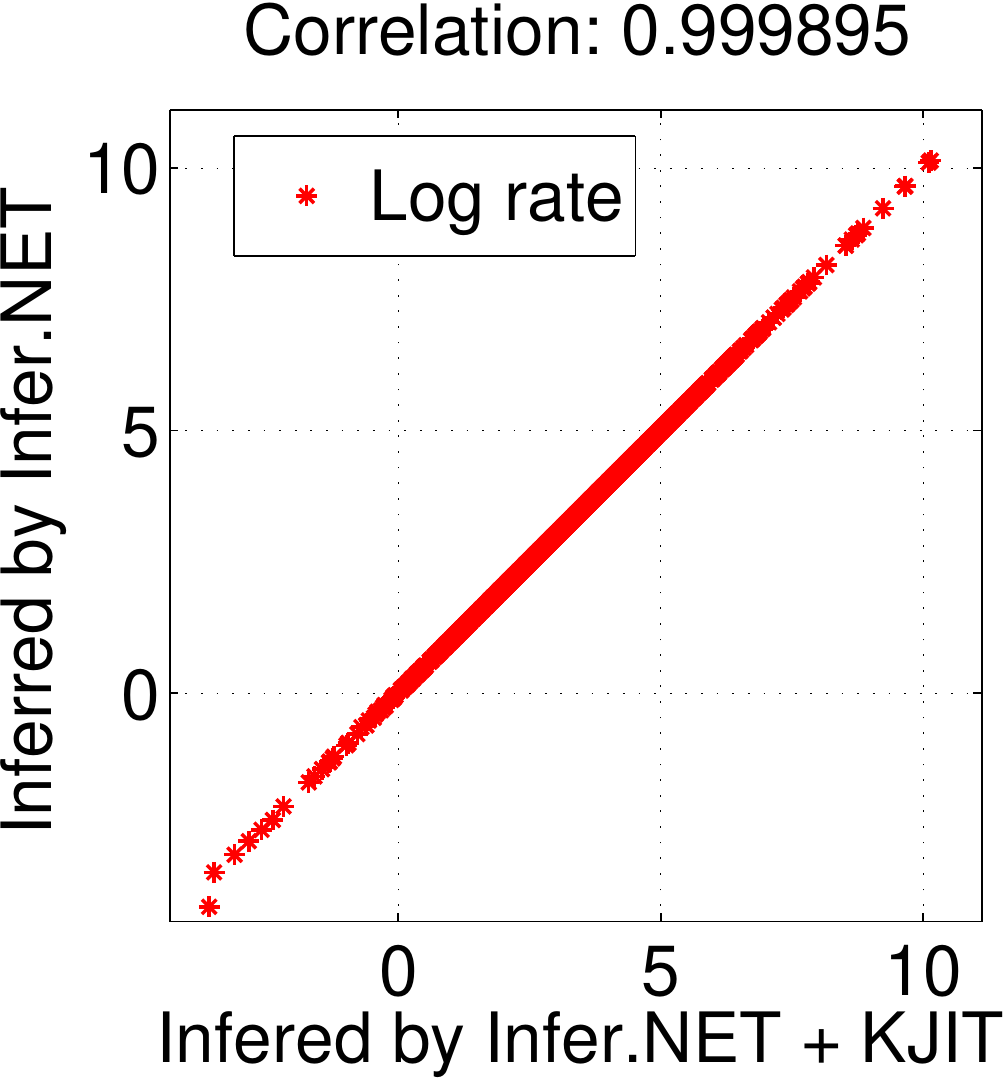}
  %\missingfigure[figwidth=0.49\columnwidth]{}
  }
  \subfloat[Inference time\label{fig:cg_infer_time}]{
  \includegraphics[width=0.33\columnwidth]{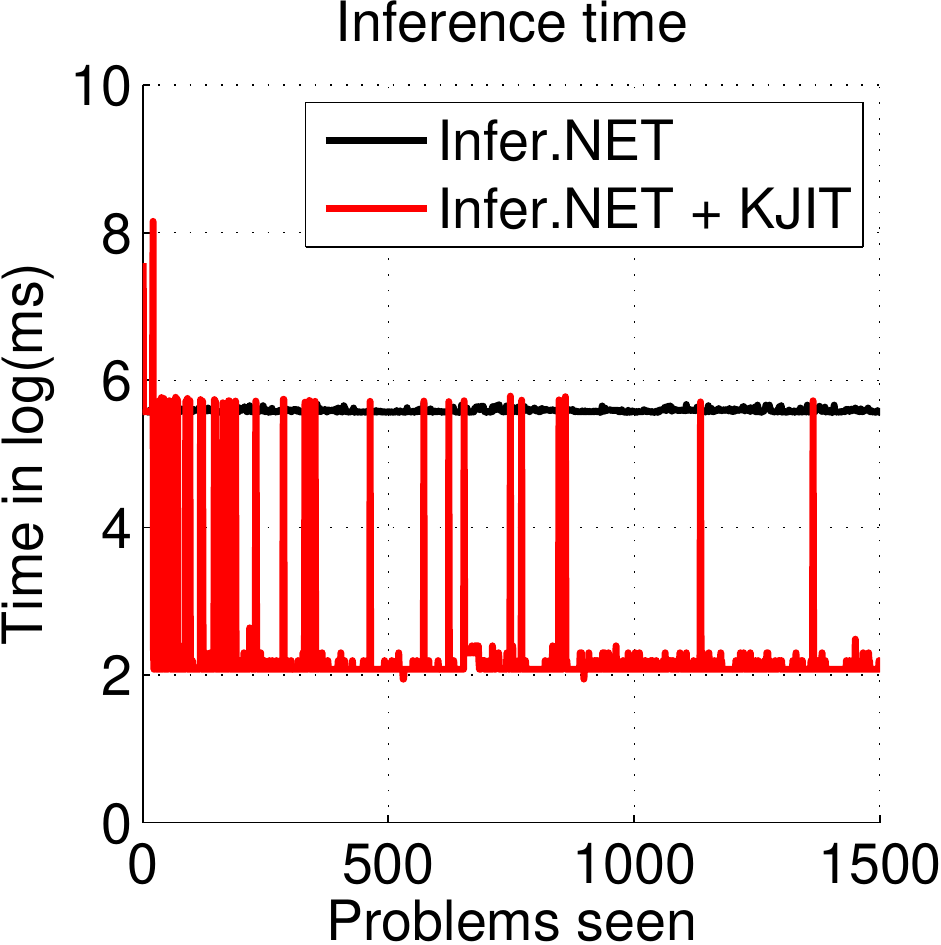}
  }
  \caption{Shape (a) and rate (b) parameters of the inferred posteriors in 
  the compound gamma problem. 
  (c) KJIT is able to infer equally good posterior parameters compared to Infer.NET, 
  while requiring a runtime several orders of magnitude lower. }
  \label{fig:cg_performance}
\end{figure}

% UCI datasets. temporal uncertainty.
\begin{figure*}[t]
\centering
\includegraphics[width=0.95\textwidth]{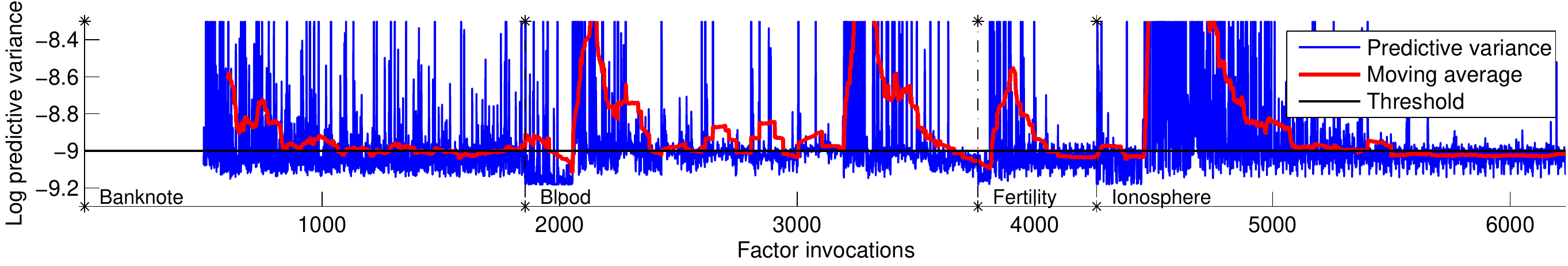}
\caption{
Uncertainty estimate of KJIT for outgoing messages on the four UCI datasets.
\label{fig:uci_temporal_uncertainty}
}
\end{figure*}

\paragraph{Experiment 4: Compound Gamma Factor} We next simulate the compound gamma factor, 
 a heavy-tailed prior distribution on the precision of a Gaussian random variable.
A variable $\tau$ is said to follow the compound gamma distribution 
if $\tau \sim \text{Gamma}(\tau; s_2, r_2)$ (shape-rate parameterization) and 
$r_2 \sim \text{Gamma}(r_2; s_1, r_1)$ where $(s_1, r_1, s_2)$ are parameters. 
The task we consider is to infer the posterior of the precision $\tau$ of a normally 
distributed variable $x \sim \mathcal{N}(x; 0, \tau)$ given realizations 
$\{x_i\}_{i=1}^n$. We consider the setting $(s_1, r_1, s_2) = (1, 1, 1)$ which was 
used in \cite{Heess2013}. Infer.NET's implementation requires two gamma factors
to specify the compound gamma. Here, we collapse them into one factor 
and let the operator learn to directly send an outgoing message $\msg{\factor}{\tau}$ 
given $\msg{\tau}{\factor}$, using Infer.NET as the oracle. 
The  default implementation of Infer.NET relies on a quadrature method.
As in \cite{Eslami2014}, we sequentially presented a number of 
problems to our algorithm, where at the beginning of each problem, a random number of observations $n$
from 10 to 100, and the parameter $\tau$, were drawn from the model.

\figref{fig:cg_infer_shape} and \figref{fig:cg_infer_rate} summarize the inferred 
posterior parameters obtained from running only Infer.NET and Infer.NET + KJIT, i.e., 
KJIT with Infer.NET as the oracle. \figref{fig:cg_infer_time} shows the inference 
time of both methods. The plots collectively show that KJIT can deliver posteriors 
in good agreement with those obtained from Infer.NET, at a much lower cost. 
Note that in this task only one message is passed to the factor in each problem.
\figref{fig:cg_infer_time} also indicates that KJIT requires fewer oracle 
consultations as more problems are seen.

%------------------------------------
%\subsection{Real data}
%------------------------------------

\paragraph{Experiment 5: Classification Benchmarks} In the final experiment, we demonstrate
that our method for learning the message operator is able to detect changes
in the distribution of incoming messages via its uncertainty estimate,
and to subsequently update its prediction through additional oracle queries.
The different distributions of incoming messages are achieved
by presenting a sequence of different classification problems to our learner.
We used four 
binary classification datasets from the UCI repository 
\citep{Lichm2013}: banknote authentication, blood transfusion, fertility 
and ionosphere, in the same binary logistic regression setting as 
before. The operator was required to learn just-in-time to send outgoing messages 
$\msg{\factor}{z_i}$ and $\msg{\factor}{p_i}$ on the four problems presented 
in sequence. The training observations consisted of 200 data points subsampled
from each dataset by stratified sampling. 
For the fertility dataset, which contains only 
100 data points, we subsampled half the points. The remaining  data were used as 
 test sets. The uncertainty threshold was set to -9, and the minibatch 
 size was 500. All other parameters were the same as in the earlier JIT learning experiment.
% in \secref{sec:logistic_factor}.

% UCI data classification error
% Inference time
\begin{figure}[ht]
  \centering
  \subfloat[Binary classification error\label{fig:uci_01_loss}]{
  \includegraphics[width=0.49\columnwidth]{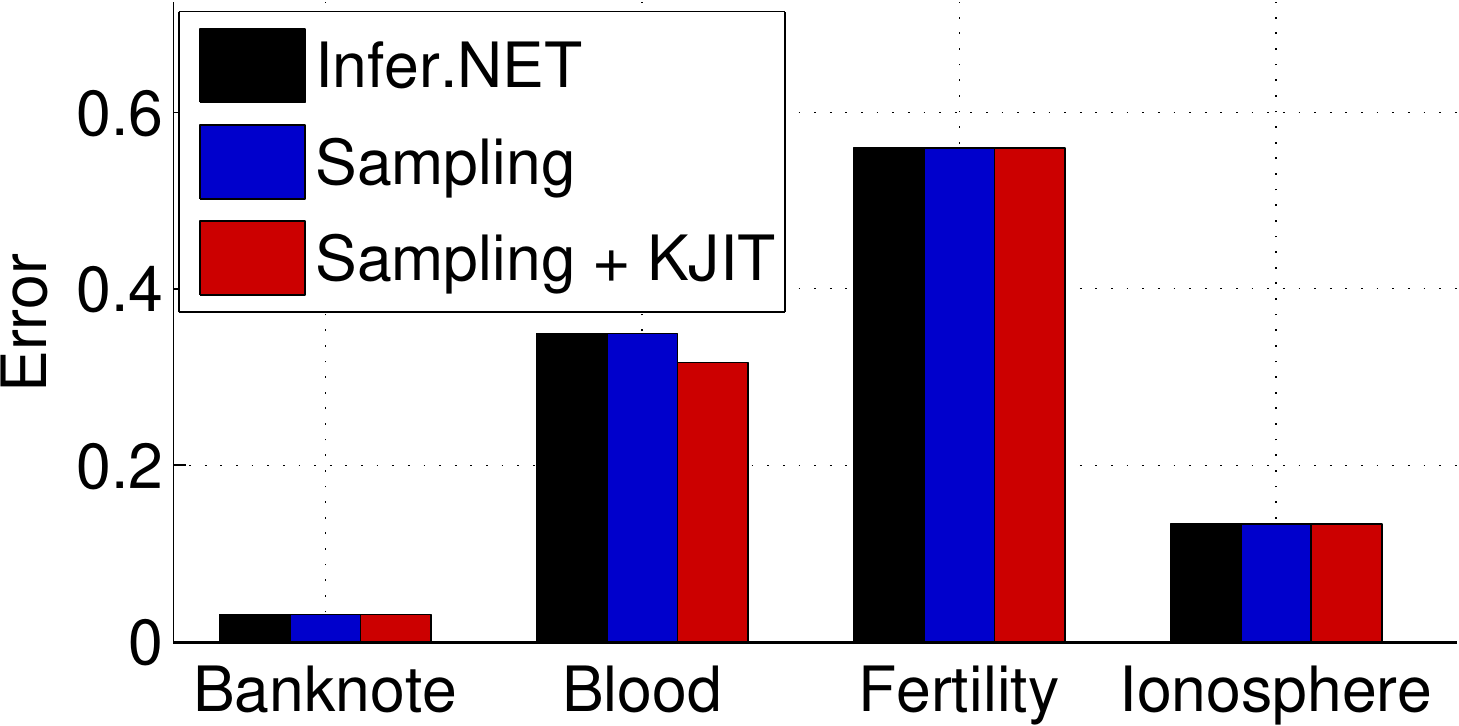}
  }
  \subfloat[Inference time\label{fig:uci_infer_time}]{
  \includegraphics[width=0.49\columnwidth]{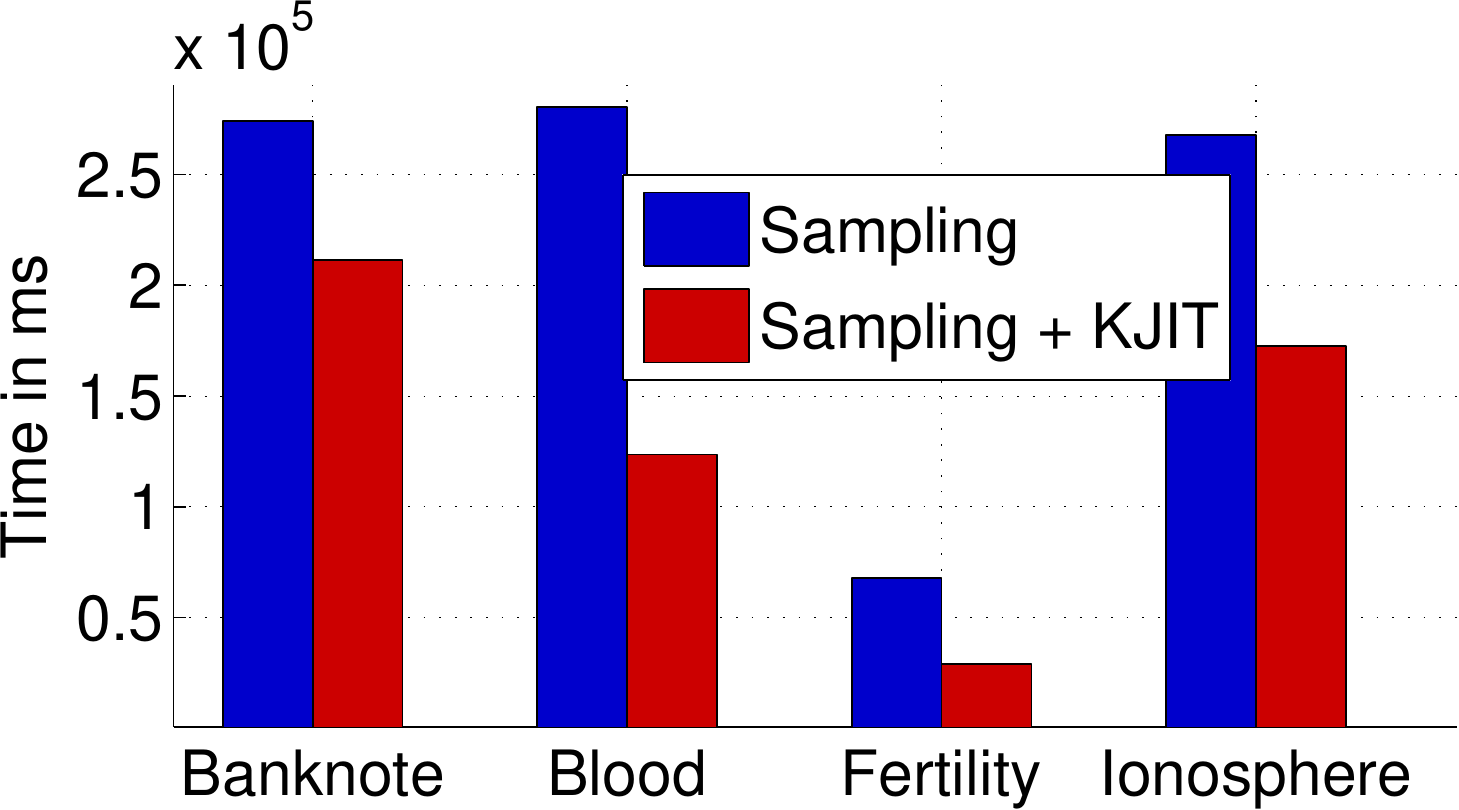}
  }
  \caption{Classification performance and inference times on the four UCI datasets. 
   }
  \label{fig:uci_performance}
\end{figure}

Classification errors on the test sets and inference times are shown in 
\figref{fig:uci_01_loss} and \figref{fig:uci_infer_time}, respectively.
The results demonstrate that KJIT improves the inference time on all the 
problems without sacrificing inference accuracy. The  predictive 
variance of each outgoing message is shown in 
\figref{fig:uci_temporal_uncertainty}. An essential feature to notice is the 
rapid increase of the uncertainty after the first EP iteration of each problem. 
As shown in \figref{fig:uci_in_msgs}, the distributions of incoming messages 
of the four problems are diverse. 
The sharp rise followed by a steady decrease of the uncertainty is a good indicator 
that the operator is able to promptly detect a change in input message distribution,
and robustly adapt to this new distribution by querying the oracle.

% \hrulefill{}
% \hl{Experiments to have}
% \begin{itemize}
% \item For predictive variance on unseen region, we may need two plots where
% in one case the training samples are concentrated in one part of the
% space, and in the second case the training samples are concentrated
% in another part. This is to ensure that the uncertainty estimate does
% not depend on the absolute location.

% \item To determine thresholds, compute the median of log predictive variance 
% on the test set split from the initial minibatch. 

%\item Choose the number of particles in the oracle importance sampler so
%that it takes roughly the same time as our operator. Compute the predictive
%quality of the two. 

% \end{itemize}

%=============================================
\section{CONCLUSIONS AND FUTURE WORK}
\label{sec:Conclusions-and-Future} 
%=============================================

We have proposed a method for learning the mapping between incoming and outgoing
messages to a factor in expectation propagation, which can be used
in place of computationally demanding Monte Carlo estimates of these updates.
Our operator has two main advantages: it can reliably evaluate the uncertainty of its prediction,
so that it only consults a more expensive oracle when it is uncertain,
and it can efficiently update its mapping online, so that it learns
from these additional consultations. Once trained, the learned mapping
performs as well as the oracle mapping, but at a far lower computational cost.
This is in large part due to a novel two-stage random feature representation of the input
messages. One topic of current research is hyperparameter selection:
at present, these are learned on an initial mini-batch of data, however
a better option would be to adapt them online as more data are seen.

\subsubsection*{ACKNOWLEDGEMENT}
We thank the anonymous reviewers for their constructive comments. 
WJ, AG, BL, and ZSz thank the Gatsby Charitable Foundation for the financial
support.

\newpage
% {\small 
\bibliographystyle{abbrvnat}
\bibliography{ref}
% }

% Place all figures before moving further
\clearpage
%%%%%%%%%%%%%%%%%%%% Appendix %%%%%%%%%%%%%%%%%%%
\newpage
\onecolumn
\appendix

% \begin{center}
% \textbf{\textcolor{black}{\LARGE{}Supplementary Materials }}
% \par\end{center}{\LARGE \par}

{\Large Kernel-Based Just-In-Time Learning for Passing Expectation Propagation
Messages}

{\large SUPPLEMENTARY MATERIAL}

%=========================================  
\section{MEDIAN HEURISTIC FOR GAUSSIAN KERNEL ON MEAN EMBEDDINGS  }
\label{sec:median_heuristic}
%=========================================

In the proposed KJIT, there are two kernels: the inner kernel $k$ for computing
mean embeddings, and the outer Gaussian kernel $\kappa$ defined on the mean
embeddings. Both of the kernels depend on a number of parameters. In this section,
we describe a heuristic to choose the kernel parameters. We emphasize that this
heuristic is merely for computational convenience. A full parameter selection
procedure like cross validation or evidence maximization will likely yield a
better set of parameters. We use this heuristic in the initial mini-batch phase
before the actual online learning.

Let $\{ r^{(l)}_{i} \mid l= 1,\ldots, c, \text{ and } i=1,\ldots, n \}$ be a set 
of $n$ incoming message tuples collected during the mini-batch phase, from $c$ variables 
neighboring the factor. Let $R_i := (r^{(l)}_i)_{l=1}^c$ be the $i^{th}$ tuple, and 
let $\mathsf{r}_i := \times_{l=1}^c r^{(l)}_i$ be the product of incoming messages 
in the $i^{th}$ tuple. Define $S_i$ and $\mathsf{s}_i$ to be the corresponding 
quantities of another tuple of messages. We will drop the subscript $i$ when considering 
only one tuple.

Recall that the kernel on two tuples of messages $R$ and $S$ is given by 
\begin{align*}
\kappa(R, S) &= \kappa(\mathsf{r}, \mathsf{s}) =
\exp\left(-\frac{\|\mu_{\mathsf{r}}-
\mu_{\mathsf{s}} \|_{\mathcal{H}}^{2}}{2\gamma^{2}}\right)  \\
  &= \exp\left(-\frac{1}{2\gamma^{2}}\left\langle \mu_{\mathsf{r}},
  \mu_{\mathsf{r}}\right\rangle +\frac{1}{\gamma^{2}}\left\langle
  \mu_{\mathsf{r}}, \mu_{\mathsf{s}}\right\rangle
  -\frac{1}{2\gamma^{2}}\left\langle
  \mu_{\mathsf{s}},\mu_{\mathsf{s}}\right\rangle \right),
\end{align*}
where 
$\langle \mu_{\mathsf{r}},\mu_{\mathsf{s}}  \rangle
= \mathbb{E}_{x \sim \mathsf{r}} \mathbb{E}_{y \sim \mathsf{s}} 
k(x-y)$. 
The inner kernel $k$ is a Gaussian kernel defined on the domain 
$\mathcal{X} := \mathcal{X}^{(1)} \times \cdots \times \mathcal{X}^{(c)} $
where $\mathcal{X}^{(l)}$ denotes the domain of $r^(l)$. For simplicity, we assume 
that $\mathcal{X}^{(l)}$ is one-dimensional. The Gaussian kernel $k$ takes
the form 
\begin{align*}
k(x-y) &= \exp\left(-\frac{1}{2}\left(x-y\right)^{\top}
\Sigma^{-1}\left(x-y\right) \right) 
= \prod_{l=1}^c \exp \left( -\frac{(x_j - y_j)^2}{2 \sigma^2_l}  \right),
\end{align*}
where $\Sigma = \diag(\sigma^2_1, \ldots, \sigma^2_c)$. The heuristic for choosing 
$\sigma^2_1, \ldots, \sigma^2_c$ and $\gamma$ is as follows. 

\begin{enumerate}
    \item Set $\sigma_l^2 := \frac{1}{n} \sum_{i=1}^n \mathbb{V}_{x_l \sim
            r_i^{(l)}}[x_l]$ where $ \mathbb{V}_{x_l \sim
                r_i^{(l)}}[x_l]$ denotes the variance of $r_i^{(l)}$.
            \item With $\Sigma = \diag(\sigma^2_1, \ldots, \sigma^2_c)$ as defined 
                in the previous step, set $\gamma^2 := \mathrm{median}\left( \{ \| \mu_{\mathsf{r}_i}
                - \mu_{\mathsf{s}_j}\|^2\}_{i,j=1}^n \right)$.
\end{enumerate}

%=========================================
\section{KERNELS AND RANDOM FEATURES}
\label{sec:kernel_random_features}
%=========================================

This section reviews relevant kernels and their random feature representations.

%------------------------------------
\subsection{RANDOM FEATURES}
%------------------------------------

This section contains a summary of \citesup{Rahimi2007}'s random
Fourier features for a translation invariant kernel. 

A kernel $k(x,y)=\left\langle \phi(x),\phi(y)\right\rangle $ in general
may correspond to an inner product in an infinite-dimensional space
whose feature map $\phi$ cannot be explicitly computed. In \citesup{Rahimi2007},
methods of computing an approximate feature maps $\hat{\phi}$ were
proposed. The approximate feature maps are such that $k(x,y)\approx\hat{\phi}(x)^{\top}\hat{\phi}(y)$
(with equality in expectation) where $\hat{\phi}(x), \hat{\phi}(y) \in\mathbb{R}^{D}$
and $D$ is the number of random features. High $D$ yields a better
approximation with higher computational cost. Assume $k(x,y)=k(x-y)$ (translation invariant)
and $x,y\in\mathbb{R}^{d}$. Random Fourier features $\hat{\phi}(x)\in\mathbb{R}^{D}$
such that $k(x,y)\approx\hat{\phi}(x)^{\top}\hat{\phi}(y)$ are generated
as follows:
\begin{enumerate}
\item Compute the Fourier transform $\hat{k}$ of the kernel $k$:
    $\hat{k}(\omega)=\frac{1}{2\pi}\int e^{-j\omega^{\top}\delta}k(\delta)\,
    d\delta$.
%For a Gaussian kernel with unit width, $\hat{k}(\omega)=\left(2\pi\right)^{-d/2}e^{-\|\omega\|^{2}/2}$.
\item Draw $D$ i.i.d. samples $\omega_{1},\ldots,\omega_{D}\in\mathbb{R}^{d}$
from $\hat{k}$. 
\item Draw $D$ i.i.d samples $b_{1},\ldots,b_{D}\in\mathbb{R}$ from $U[0,2\pi]$
(uniform distribution).
\item $\hat{\phi}(x)=\sqrt{\frac{2}{D}}\left(\cos\left(\omega_{1}^{\top}x+b_{1}\right),\ldots,\cos\left(\omega_{D}^{\top}x+b_{D}\right)\right)^{\top}\in\mathbb{R}^{D}$
\end{enumerate}

\paragraph{Why It Works }
\begin{thm}
Bochner's theorem \citepsup{Rudin2013}. A continuous kernel $k(x,y)=k(x-y)$ on $\mathbb{R}^{m}$
is positive definite iff $k(\delta)$ is the Fourier transform of
a non-negative measure.
\end{thm}
Furthermore, if a translation invariant kernel $k(\delta)$ is properly
scaled, Bochner's theorem guarantees that its Fourier transform $p(\omega)$
is a  probability distribution. From this fact, we have 
\[
k(x-y)=\int\hat{k}(\omega)e^{j\omega^{\top}\left(x-y\right)}\, d\omega=\mathbb{E}_{\omega}\left[\eta_{\omega}(x)\eta_{\omega}(y)^{*}\right], 
\]
where $j=\sqrt{-1}$, $\eta_{\omega}(x)=e^{j\omega^{\top}x}$ and $\cdot^{*}$ denotes
the complex conjugate. Since both $\hat{k}$ and $k$ are real, the complex
exponential contains only the cosine terms. Drawing $D$ samples lowers the variance of the approximation.
\begin{thm}
\label{thm:Separation-of-variables.}Separation of variables. Let
$\hat{f}$ be the Fourier transform of $f$. If $f(x_{1},\ldots,x_{d})=f_{1}(x_{1})\cdots f_{d}(x_{d})$,
then $\hat{f}(\omega_{1},\ldots,\omega_{d})=\prod_{i=1}^{d}\hat{f}_{i}(\omega_{i})$. 
\end{thm}
Theorem \ref{thm:Separation-of-variables.} suggests that the random
Fourier features can be extended to a product kernel by drawing $\omega$
independently for each kernel. 

%------------------------------------
\subsection{MV (MEAN-VARIANCE) KERNEL }
\label{sec:mv_kernel}
%------------------------------------

Assume there are $c$ incoming messages $R:=\left(r^{(l)}\right)_{l=1}^{c}$ 
and $S:=\left(s^{(l)}\right)_{l=1}^{c}$
. Assume that
\begin{align*}
\mathbb{E}_{r^{(l)}}\left[x\right] & =m_{l}\\
\mathbb{V}_{r^{(l)}}\left[x\right] & =v_{l}\\
\mathbb{E}_{s^{(l)}}\left[y\right] & =\mu_{l}\\
\mathbb{V}_{s^{(l)}}\left[y\right] & =\sigma_{l}^{2}.
\end{align*}
Incoming messages are not necessarily Gaussian. 
The MV (mean-variance) kernel is defined as a product
kernel on means and variances. 
\begin{align*}
\kappa_{\text{mv}}\left(R, S\right) 
%&=\prod_{i=1}^{c}\kappa^{(i)}\left(p^{(i)},q^{(i)}\right)\\
 &=\prod_{i=1}^{c}k\left(\left(m_{i}-\mu_{i}\right)/w_{i}^{m}\right)
 \prod_{i=1}^{c}k\left(\left(v_{i}-\sigma_{i}^{2}\right)/w_{i}^{v}\right),
\end{align*}
where $k$ is a Gaussian kernel with unit width. The kernel $\kappa_{\text{mv}}$
has $P:=\left(w_{1}^{m},\ldots,w_{c}^{m},w_{1}^{v},\ldots,w_{c}^{v}\right)$
as its parameters. With this kernel, we treat messages as finite dimensional
vectors. All incoming messages $(s^{(i)})_{i=1}^{c}$ are 
represented as $\left(\mu_{1},\ldots,\mu_{c},\sigma_{1}^{2},\ldots,\sigma_{c}^{2}\right)^{\top}$.
This treatment reduces the problem of having distributions as inputs
to the familiar problem of having input points from a Euclidean space.
The random features of \citesup{Rahimi2007} can be applied straightforwardly. 

%------------------------------------
\subsection{EXPECTED PRODUCT KERNEL}
\label{sub:Expected-Product-Kernel}
%------------------------------------
Given two distributions $r(x)=\mathcal{N}(x;m_{r},V_{r})$ and $s(y)=\mathcal{N}(y;m_{s},V_{s})$
($d$-dimensional Gaussian), the expected product kernel is defined
as 
\[
\kappa_{\text{pro}}(r, s)=\left\langle \mu_{r},\mu_{s}\right\rangle _{\mathcal{H}}=\mathbb{E}_{r}\mathbb{E}_{s}k(x-y),
\]
where $\mu_{r}:=\mathbb{E}_{r}k(x,\cdot)$ is the mean embedding of
$r$, and we assume that the kernel $k$ associated with $\mathcal{H}$
is translation invariant i.e., $k(x,y)=k(x-y)$. The goal here is
to derive random Fourier features for the expected product kernel.
That is, we aim to find $\hat{\phi}$ such that $\kappa_{\text{pro}}(r, s)\approx\hat{\phi}(r)^{\top}\hat{\phi}(s)$
and $\hat{\phi}\in\mathbb{R}^{D}$.

We first give some results which will be used to derive the Fourier
features for inner product of mean embeddings.
\begin{lem}
If $b\sim\mathcal{N}(b;0,\sigma^{2})$, then $\mathbb{E}[\cos(b)]=\exp\left(-\frac{1}{2}\sigma^{2}\right)$.\label{lemma:e_cos}\end{lem}
\begin{proof}
We can see this by considering the characteristic function of $x\sim\mathcal{N}(x;\mu,\sigma^{2})$
which is given by 
\[
c_{x}(t)=\mathbb{E}_{x}\left[\exp\left(itb\right)\right]
=\exp\left(itm-\frac{1}{2}\sigma^{2}t^{2} \right).
\]
For $m=0,t=1$, we have 
\[
c_{b}(1)=\mathbb{E}_{b}\left[\exp(ib)\right]=
\exp\left(-\frac{1}{2}\sigma^{2}\right)=\mathbb{E}_{b}\left[\cos(b)\right],
\]
where the imaginary part  $i\sin(tb)$ vanishes.
\end{proof}

From \citesup{Rahimi2007} which provides random features for $k(x-y)$,
we immediately have
\begin{align*}
\mathbb{E}_{r}\mathbb{E}_{s}k(x-y) & \approx\mathbb{E}_{r}\mathbb{E}_{s}\frac{2}{D}\sum_{i=1}^{D}\cos\left(w_{i}^{\top}x+b_{i}\right)\cos\left(w_{i}^{\top}y+b_{i}\right)\\
 & =\frac{2}
 {D}\sum_{i=1}^{D}\mathbb{E}_{r(x)}\cos\left(w_{i}^{\top}x+b_{i}\right)\mathbb{E}_{s(y)}\cos\left(w_{i
 }^{\top}y+b_{i}\right),
\end{align*}
where $\{w_{i}\}_{i=1}^{D}\sim\hat{k}(w)$ (Fourier transform of $k$)
and $\{b_{i}\}_{i=1}^{D}\sim U\left[0,2\pi\right]$. 

Consider $\mathbb{E}_{r(x)}\cos\left(w_{i}^{\top}x+b_{i}\right)$.
Define $z_{i}=w_{i}^{\top}x+b_{i}$. So $z_{i}\sim\mathcal{N}(z_{i};w_{i}^{\top}m_{r}+b_{i},w_{i}^{\top}V_{r}w_{i})$.
Let $d_{i}\sim\mathcal{N}(0,w_{i}^{\top}V_{r}w_{i})$. Then, $r(d_{i}+w_{i}^{\top}m_{r}+b_{i})=\mathcal{N}(w_{i}^{\top}m_{r}+b_{i},w_{i}^{\top}V_{r}w_{i})$
which is the same distribution as that of $z_{i}$. From these definitions
we have,
\begin{align*}
\mathbb{E}_{r(x)}\cos\left(w_{i}^{\top}x+b_{i}\right) & =\mathbb{E}_{r(z_{i})}\cos(z_{i})\\
 & =\mathbb{E}_{r(d_{i})}\cos\left(d_{i}+w_{i}^{\top}m_{r}+b_{i}\right)\\
 & \overset{(a)}
 {=}\mathbb{E}_{r(d_{i})}\cos(d_{i})\cos(w_{i}^{\top}m_{r}+b_{i})-\mathbb{E}_{r(d_{i})}\sin(d_{i})\sin
 (w_{i}^{\top}m_{r}+b_{i})\\
 & \overset{(b)}{=}\cos(w_{i}^{\top}m_{r}+b_{i})\mathbb{E}_{r(d_{i})}\cos(d_{i})\\
 & \overset{(c)}{=}\cos(w_{i}^{\top}m_{r}+b_{i})
 \exp\left(-\frac{1}{2}w_{i}^{\top}V_{r}w_{i}\right),
\end{align*}
where at $(a)$ we use $\cos(\alpha+\beta)=\cos(\alpha)\cos(\beta)-\sin(\alpha)\sin(\beta)$.
We have $(b)$ because $\sin$ is an odd function and $\mathbb{E}_{r(d_{i})}\sin(d_{i})=0$.
The last equality $(c)$ follows from Lemma \ref{lemma:e_cos}. It
follows that the random features $\hat{\phi}(r)\in\mathbb{R}^{D}$
are given by
\[
\hat{\phi}(r)=\sqrt{\frac{2}{D}}\left(\begin{array}{c}
\cos(w_{1}^{\top}m_{r}+b_{1})\exp\left(-\frac{1}{2}w_{1}^{\top}V_{r}w_{1}\right)\\
\vdots\\
\cos(w_{D}^{\top}m_{r}+b_{D})\exp\left(-\frac{1}{2}w_{D}^{\top}V_{r}w_{D}\right)
\end{array}\right).
\]

Notice that the translation invariant kernel $k$
provides $\hat{k}$ from which $\{w_{i}\}_{i}$ are
drawn. For a different type of distribution $r$, we only need to
be able to compute $\mathbb{E}_{r(x)}\cos\left(w_{i}^{\top}x+b_{i}\right)$.
With $\hat{\phi}(r)$, we have $\kappa_\text{pro}(r, s)\approx\hat{\phi}(r)^{\top}\hat{\phi}(s)$
with equality in expectation.

\paragraph{Analytic Expression for Gaussian Case}

For reference, if $r, s$ are normal distributions and $k$
is a Gaussian kernel, an analytic expression is available. Assume
$k(x-y)=\exp\left(-\frac{1}{2}\left(x-y\right)^{\top}\Sigma^{-1}\left(x-y\right)\right)$
where $\Sigma$ is the kernel parameter. Then
\begin{align*}
\mathbb{E}_{r}\mathbb{E}_{s}k(x-y) & =\sqrt{\frac{\det(D_{rs})}{\det(\Sigma^{-1})}}\exp\left(-\frac{1}{2}\left(m_{r}-m_{s}\right)^{\top}D_{rs}\left(m_{r}-m_{s}\right)\right),\\
D_{rs} & :=\left(V_{r}+V_{s}+\Sigma\right)^{-1}.
\end{align*}

\paragraph{Approximation Quality}

\begin{comment}
See \verb|primalKEGauss.m|.
\end{comment}
{} The following result compares the randomly generated features to
the true kernel matrix using various numbers of random features $D$. 
For each $D$, we repeat 10 trials, where the randomness in each trial 
arises from the construction of the random features. 
Samples are univariate normal distributions $\mathcal{N}(m, v)$ where 
$m \sim \mathcal{N}(0, 9)$ and $v \sim \text{Gamma}(3, 1/4)$ (shape-rate 
parameterization). The kernel parameter was $\Sigma:=\sigma^2 I$ where $\sigma^2=3$.

\begin{center}
\includegraphics[width=10cm]{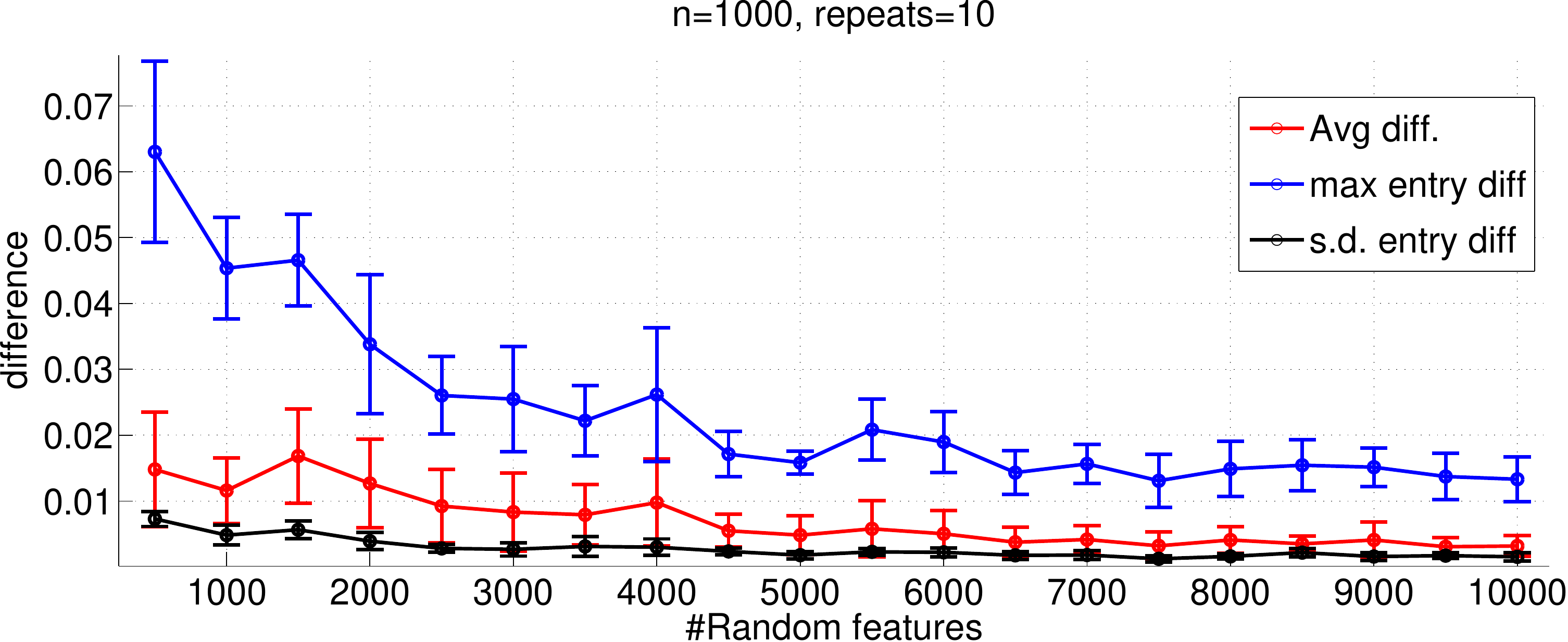}
\end{center}

``max entry diff''
refers to the maximum entry-wise difference between the true kernel
matrix and the approximated kernel matrix.

%------------------------------------
\subsection{PRODUCT KERNEL ON MEAN EMBEDDINGS}
%------------------------------------
Previously, we have defined an expected product kernel on single distributions.
One way to define a kernel between two tuples of more than one incoming message is to take 
a product of the kernels defined on each message.

Let $\mu_{r^{(l)}}:=\mathbb{E}_{r^{(l)}(a)}k^{(l)}(\cdot,a)$ be the mean
embedding \citepsup{Smola2007} of the distribution $r^{(l)}$ into
RKHS $\mathcal{H}^{(l)}$ induced by the kernel $k$. Assume $k^{(l)}=k_{\text{gauss}}^{(l)}$
(Gaussian kernel) and assume there are $c$ incoming messages $R:=(r^{(i)}(a^{(i)}))_{i=1}^{c}$
and $S:=(s^{(i)}(b^{(i)}))_{i=1}^{c}$. A product of expected product
kernels is defined as 
\begin{align*}
\kappa_{\text{pro, prod}}(R, S) & :=\left\langle \bigotimes_{l=1}^{c}\mu_{r^{(l)}},
\bigotimes_{l=1}^{c}\mu_{s^{(l)}}\right\rangle _{\otimes_{l}\mathcal{H}^{(l)}} \\ 
&= 
\prod_{l=1}^{c}\mathbb{E}_{r^{(l)}(a^{(l)})}\mathbb{E}_{s^{(l)}(b^{(l)})}k_{\text{gauss}}^{(l)} 
\left(a^{(l)}, b^{(l)} \right)\approx\hat{\phi}(R)^{\top}\hat{\phi}(S),
\end{align*}
where $\hat{\phi}(R)^{\top}\hat{\phi}(S)=\prod_{l=1}^{c}
\hat{\phi}^{(l)}(r^{(l)})^{\top}\hat{\phi}^{(l)}(s^{(l)})$.
The feature map $\hat{\phi}^{(l)}(r^{(l)})$ can be estimated by applying
the random Fourier features to $k_{\text{gauss }}^{(l)}$and taking
the expectations $\mathbb{E}_{r^{(l)}(a)}\mathbb{E}_{s^{(l)}(b)}$.
The final feature map is $\hat{\phi}(R)=\hat{\phi}^{(1)}(r^{(1)})\circledast\hat{\phi}^{(2)}(r^{(2)})\circledast\cdots\circledast\hat{\phi}^{(c)}(r^{(c)})\in\mathbb{R}^{D^{c}}$, where
$\circledast$ denotes a Kronecker product and we assume that $\hat{\phi}^{(l)}\in\mathbb{R}^{D}$
for $l\in\{1,\ldots,c\}$. 

%------------------------------------
\subsection{SUM KERNEL ON MEAN EMBEDDINGS}
%------------------------------------

If we instead define the kernel as the sum of $c$ kernels, we have
\begin{align*}
\kappa_{\text{pro, sum}}(R, S) & =\sum_{l=1}^{c}\left\langle \mu_{r^{(l)}},\mu_{s^{(l)}}\right\rangle _{\mathcal{H}^{(l)}}\\
 & \approx\sum_{l=1}^{c}\hat{\phi}^{(l)}(r^{(l)})^{\top}\hat{\phi}^{(l)}(s^{(l)})\\
 & =\hat{\varphi}(R)^{\top}\hat{\varphi}(S),
\end{align*}
where $\hat{\varphi}(R):=\left(\hat{\phi}^{(1)}(r^{(1)})^{\top},\ldots,\hat{\phi}^{(c)}
(r^{(c)})^{\top}\right)^{\top}\in\mathbb{R}^{cD}$.

%------------------------------------
\subsection{NUMBER OF RANDOM FEATURES FOR GAUSSIAN KERNEL ON MEAN EMBEDDINGS }
%------------------------------------
We quantify the effect of $D_\mathrm{in}$ and $D_\mathrm{out}$
empirically as follows. We generate 300 Gaussian messages, compute the true Gram
matrix and the approximate Gram matrix given by the random features,
and report the Frobenius norm of the difference of the two matrices
on a grid of $D_\mathrm{in}$ and $D_\mathrm{out}$. For each $(D_\mathrm{in},D_\mathrm{out})$,
we repeat 20 times with a different set of random features and report
the averaged Frobenius norm.

\begin{center}
\includegraphics[width=8cm]{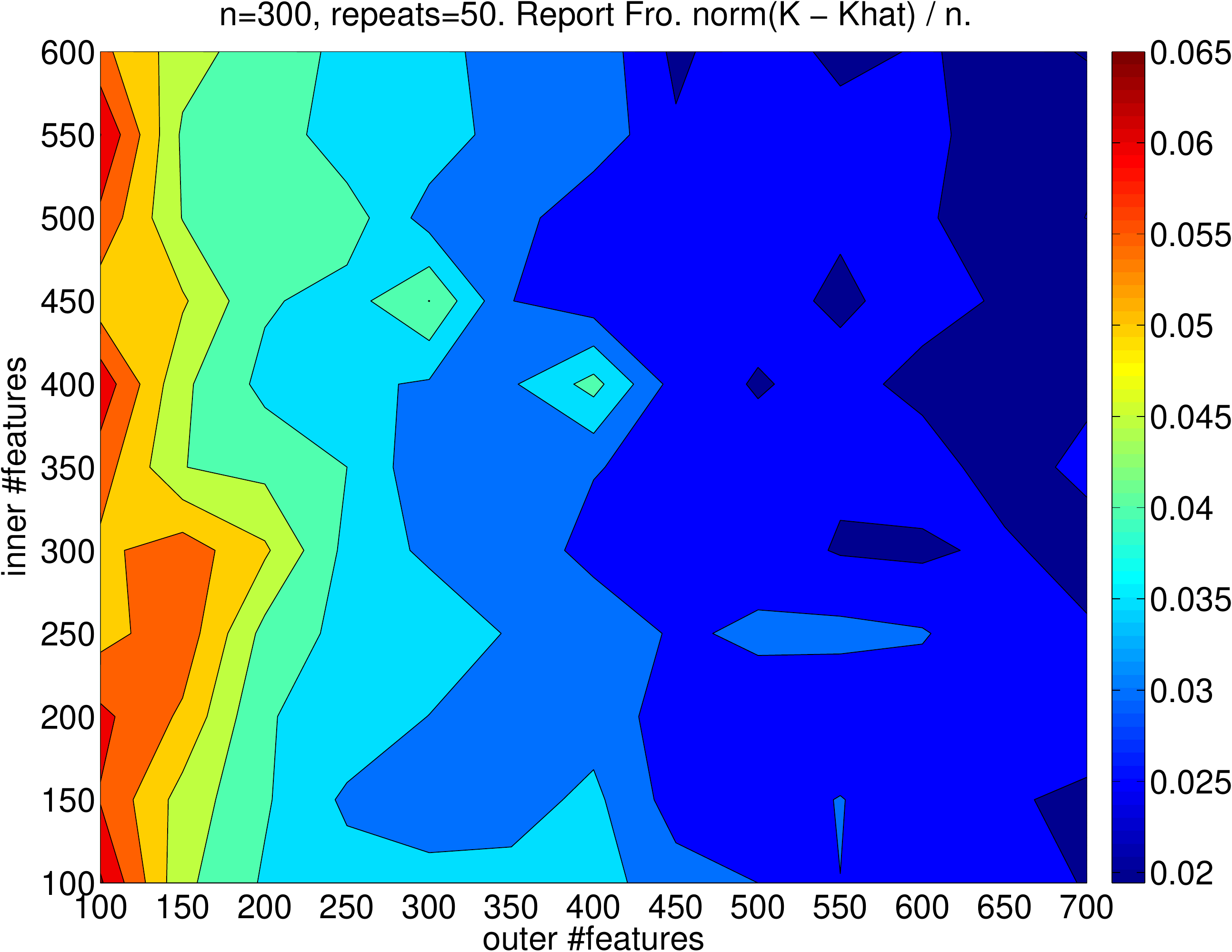} 
\end{center}

The result suggests that $D_\mathrm{out}$ has more effect in improving the
approximation.

\section{MORE DETAILS ON EXPERIMENT 1: BATCH LEARNING }
\label{sec:batch_learning_detail}
%=========================================

There are a number of kernels on distributions we may use for just-in-time learning.
To find the most suitable kernel, we compare the performance of each on a 
collection of incoming and output
messages at the logistic factor in the binary logistic regression problem i.e., 
same problem as in experiment $1$ in the main text.
All messages are collected by running EP 20 times on generated toy data. 
Only messages in the first five iterations are considered as messages passed in 
the early phase of EP vary more than in a near-convergence phase.
The regression output to be learned is the numerator of \eqref{eq:msgPassing:EP}. 

A training set of 5000 messages and a test set of 3000 messages are obtained by 
subsampling all the collected messages. 
Where random features are used, kernel widths 
and regularization parameters are chosen by leave-one-out cross validation.
To get a good sense of the approximation error from the random features, we also 
compare with incomplete Cholesky factorization (denoted by IChol), a widely used 
Gram matrix approximation technique. We use
hold-out cross validation with randomly chosen training and validation sets
for parameter selection,
%repeated hold-out for parameter selection 
and kernel ridge regression in its dual form when the incomplete Cholesky 
factorization is used. 

Let $\factor$ be the logistic factor and $m_{\factor \rightarrow i}$ be an outgoing message. 
Let $q_{\factor \rightarrow i}$ be the ground truth belief message (numerator) associated with 
$m_{\factor\rightarrow i}$. The error metric we use is 
$\mathrm{KL}[q_{\factor \rightarrow i}\,||\, \hat{q}_{\factor \rightarrow i}]$
where  $\hat{q}_{\factor \rightarrow i}$ are the belief messages estimated by a learned 
regression function. The following table reports the mean of the log KL-divergence 
and standard deviations.

\begin{center}

\begin{tabular}{lcc}
\hline
& \textbf{mean log KL} & \textbf{s.d. of log KL}  \\\hline
 Random features + MV Kernel & -6.96 & 1.67 \\ 
 {Random features + Expected product kernel on joint embeddings} & -2.78 & 1.82  \\
 {Random features + Sum of expected product kernels} & -1.05 & 1.93  \\
 {Random features + Product of expected product kernels} & -2.64 & 1.65  \\
 \textbf{Random features + Gaussian kernel on joint embeddings} (KJIT) & -8.97 & 1.57  \\
 {IChol + sum of Gaussian kernel on embeddings} & -2.75 & 2.84  \\
 IChol + Gaussian kernel on joint embeddings & -8.71 & 1.69  \\\hline
 Breiman's random forests \citepsup{Breiman2001} & -8.69 & 1.79\\
 Extremely randomized trees \citepsup{Geurts2006} & -8.90 & 1.59 \\
 \citetsup{Eslami2014}'s random forests \citepsup{Eslami2014} & -6.94 & 3.88 \\
 \hline
\end{tabular}
\end{center}

The MV kernel is defined in \secref{sec:mv_kernel}. 
Here product (sum) of expected product kernels refers to a product (sum) of kernels, where each 
is an expected product kernel defined on one incoming message. Evidently, the Gaussian 
kernel on joint mean embeddings performs significantly better than other kernels. 
Besides the proposed method, we also compare the message prediction performance
to  Breiman's random forests \citepsup{Breiman2001}, extremely randomized trees
\citepsup{Geurts2006}, and \citetsup{Eslami2014}'s random forests. We use
scikit-learn 
toolbox for the extremely randomized trees and Breiman's random forests. For
\citetsup{Eslami2014}'s random forests, we reimplemented the method as
closely as possible according to the description given in \citetsup{Eslami2014}.  
In all cases, the number of trees is set to 64.
Empirically we observe that decreasing the
log KL error below -8 will not yield 
a noticeable performance gain in the actual EP. 

%\begin{comment} 
%\textbf{Result paths:} 
%\begin{itemize}
%\item{ \small\verb|exp7/RFGMVMapperLearner_binlogis_bw_proj_n400_iter5_sf1_st20_ntr5000.mat| } 
%\item{ \small\verb|exp7/RFGJointEProdLearner_binlogis_bw_proj_n400_iter5_sf1_st20_ntr5000.mat| } 
%\item{ \small\verb|exp7/RFGSumEProdLearner_binlogis_bw_proj_n400_iter5_sf1_st20_ntr5000.mat| } 
%\item{ \small\verb|exp7/RFGProductEProdLearner_binlogis_bw_proj_n400_iter5_sf1_st20_ntr5000.mat| } 
%\item{ \small\verb|exp7/RFGJointKGGLearner_binlogis_bw_proj_n400_iter5_sf1_st20_ntr5000.mat| } 
%\item{ \small\verb|exp7/ICholKProduct_binlogis_bw_proj_n400_iter5_sf1_st20_ntr5000.mat| } 
%\item{ \small\verb|exp7/ICholKProduct_binlogis_bw_proj_n400_iter5_sf1_st20_ntr5000.mat| } 
%\item{ \small\verb|exp7/ICholKSum_binlogis_bw_proj_n400_iter5_sf1_st20_ntr5000.mat| } 
%\item{ \small\verb|exp7/ICholKGGaussianJoint_binlogis_bw_proj_n400_iter5_sf1_st20_ntr5000.mat| } 
%\end{itemize}

%\textbf{Selected parameters:} 
%\begin{itemize}
%\item{ \footnotesize\verb|GenericMapper(CondFMFiniteOut(RandFourierGaussMVMap(mw2s=[1.11211      65.6983], vw2s=[0.001     0.52568])), DNormalLogVarBuilder)| } 
%\item{ \footnotesize\verb|GenericMapper(CondFMFiniteOut(RFGJointEProdMap(gw2s=[0.11111      4.2764])), DNormalLogVarBuilder)| } 
%\item{ \footnotesize\verb|GenericMapper(CondFMFiniteOut(RFGSumEProdMap(gw2s=[0.555556      21.3243])), DNormalLogVarBuilder)| } 
%\item{ \footnotesize\verb|GenericMapper(CondFMFiniteOut(RFGProductEProdMap(gw2s=[0.011111     0.44049])), DNormalLogVarBuilder)| } 
%\item{ \footnotesize\verb|GenericMapper(StackInstancesMapper(BayesLinRegFM(RFGJointKGG(embed_w2=[0.055556      2.1709], outer_w2=0.26)), BayesLinRegFM(RFGJointKGG(embed_w2=[0.055556      2.1709], outer_w2=0.10))), DNormalLogVarBuilder)| } 
%\item{ \footnotesize\verb|GenericMapper(StackInstancesMapper(CondCholFiniteOut(r=72), CondCholFiniteOut(r=34)), DNormalLogVarBuilder)| } 
%\item{ \footnotesize\verb|GenericMapper(StackInstancesMapper(CondCholFiniteOut(r=72), CondCholFiniteOut(r=34)), DNormalLogVarBuilder)| } 
%\item{ \footnotesize\verb|GenericMapper(StackInstancesMapper(CondCholFiniteOut(r=90), CondCholFiniteOut(r=119)), DNormalLogVarBuilder)| } 
%\item{ \footnotesize\verb|GenericMapper(StackInstancesMapper(CondCholFiniteOut(r=158), CondCholFiniteOut(r=225)), DNormalLogVarBuilder)| } 
%\end{itemize}

%ans =

%   -6.9554    1.6726       NaN
%   -2.7765    1.8261       NaN
%   -1.0518    1.9315       NaN
%   -2.6410    1.6450       NaN
%   -8.9591    1.6218       NaN
%   -3.2097    1.9868       NaN
%   -3.2097    1.9868       NaN
%   -2.7510    2.8382       NaN
%   -8.7144    1.6864       NaN
%\end{comment}

%------------------------------------
%\subsection{UNCERTAINTY ESTIMATES}
%------------------------------------
To verify that the uncertainty estimates given by KJIT coincide with the actual 
predictive performance (i.e., accurate prediction when confident), we plot the 
predictive variance against the KL-divergence error on both the training and
test sets. The results are shown in \figref{fig:logistic_predvar_g}.
The uncertainty estimates show a positive correlation with the KL-divergence errors. 
It is instructive to note that no point lies at the bottom right i.e., making a
large error while being confident. 
The fact that the errors on the training set are roughly the same as the errors
on the test set indicates that the operator does not overfit.

\begin{figure}[t]
  \centering
  \includegraphics[width=8cm]{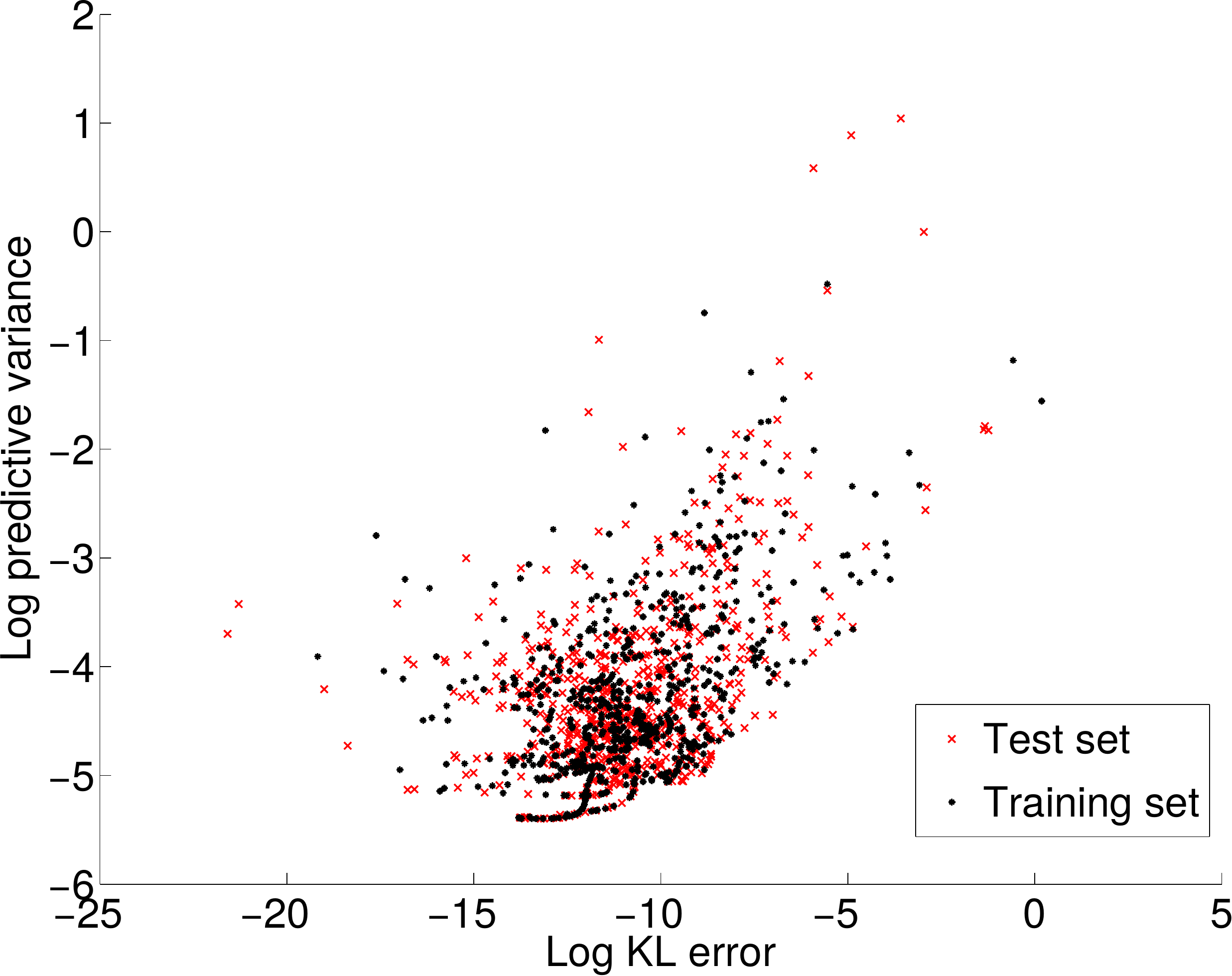}
  \caption{KL-divergence error versus predictive variance for predicting the 
  mean of $m_{f \rightarrow z_i}$ (normal distribution) in the logistic factor problem. }
  \label{fig:logistic_predvar_g}
\end{figure}

\bibliographystylesup{abbrvnat}
\bibliographysup{refappendix} 
\end{document}